\documentclass{ieeetj}

\usepackage{psfrag}
\usepackage{amsmath,amssymb,amsfonts,bbm,nicefrac,mathtools,lipsum}
\usepackage{latexsym}
\usepackage{graphicx}
\usepackage{empheq}
\usepackage{algorithm}
\usepackage{algpseudocode}
\usepackage{mdwlist}
\usepackage{balance}
\usepackage{colortbl}
\usepackage{epstopdf}
\usepackage{float}
\usepackage{hyperref}
\hypersetup{hidelinks=true}
\usepackage{booktabs,tabularx}
\usepackage{amsthm}
\usepackage{multirow}
\usepackage{overpic}
\usepackage{DejaVuSans}
\usepackage{dblfloatfix}
\usepackage{cite}
\usepackage{soul}
\usepackage{graphicx}

\def\BibTeX{{\rm B\kern-.05em{\sc i\kern-.025em b}\kern-.08em
    T\kern-.1667em\lower.7ex\hbox{E}\kern-.125emX}}
\AtBeginDocument{\definecolor{tmlcncolor}{cmyk}{0.93,0.59,0.15,0.02}\definecolor{NavyBlue}{RGB}{0,86,125}}

\def\OJlogo{\vspace{-4pt}$<$Society logo(s) and publication title will appear here.$>$}
\def\seclogo{\vspace{10pt}$<$Society logo(s) and publication title will appear here.$>$}

\def\authorrefmark#1{\ensuremath{^{\textbf{#1}}}}

\newtheorem{theorem}{Theorem}

\newcounter{problem}
\renewcommand{\theproblem}{P\arabic{problem}}

\newenvironment{problem}
  {\refstepcounter{problem}%
   \renewcommand{\theequation}{\theproblem}%
   \align}
  {\endalign}

\graphicspath{{./figures/}}

\begin{document}

\bstctlcite{BSTcontrol}
\receiveddate{XX Month, XXXX}
\reviseddate{XX Month, XXXX}
\accepteddate{XX Month, XXXX}
\publisheddate{XX Month, XXXX}
\currentdate{XX Month, XXXX}
\doiinfo{XXXX.2022.1234567}

\markboth{Resource-Efficient Distributed Recursive Gaussian Processes}{King et. al.}

\author{Josephine King \authorrefmark{1}, Student Member, IEEE \\ Ali Emre Balc\i \authorrefmark{1}, Student Member, IEEE \\ Raj Thilak Rajan \authorrefmark{1}, Senior Member, IEEE}
\affil{Signal Processing Systems, Faculty of EEMCS, Delft University of Technology, Mekelweg 4, 2628CD, Delft, The Netherlands}
\corresp{Corresponding author: Ali Emre Balc\i \  (a.e.balci@tudelft.nl)}
\authornote{This work is partially funded by the Sensor AI Lab through the AI Labs Program of the Delft University of Technology, and EU-HORIZON-KDT-JU-2023-2-RIA
  under grant agreement No 101139996, the ShapeFuture project.}

\title{Resource-Efficient Distributed Recursive Gaussian Processes}

\begin{abstract}
  Gaussian processes (GPs) provide a flexible framework for learning unknown functions from noisy measurements while quantifying predictive uncertainty, making them well suited for estimation in multi-agent systems. However, when measurements are collected by multiple agents, maintaining a unified GP model without centralized processing requires efficient distributed algorithms that can operate using local measurements and communication with neighboring agents. In this work, we develop two distributed recursive GP (RGP) algorithms for multi-output GP regression: ADMM-RGP and PDMM-RGP. We analyze the stability and convergence of both algorithms and develop parameter selection strategies to accelerate convergence, thus reducing the communication burden. The proposed methods are validated on a real-world multi-output wind dataset, and their convergence behavior is examined across communication graphs with varying connectivity. Numerical experiments demonstrate that ADMM-RGP and PDMM-RGP can significantly reduce communication relative to the state of the art, while maintaining comparable estimation accuracy and network-wide consensus.
\end{abstract}


\maketitle

\section{Introduction}

\IEEEPARstart{M}{ulti} agent systems have gained significant attention in recent years across a variety of applications, including search-and-rescue \cite{darmanin_review_2017}, agricultural monitoring \cite{davoodi_graph_2020}, smart grids \cite{mahela_comprehensive_2022}, and logistics \cite{dorri_multi-agent_2018}. A key advantage of these systems is their ability to simultaneously collect data from multiple locations \cite{kontoudis_scalable_2024}. Exploiting these measurements requires agents to combine locally acquired information to estimate an underlying quantity of interest, such as an environmental field or an unknown system function. In many applications, measurements arrive sequentially and are spatially distributed across the network, motivating learning methods that can update their estimates online while operating with local computation and sparse communication \cite{ding_resource-efficient_2024, jakkala_multi-robot_2024}.

Gaussian processes (GPs) are a Bayesian supervised machine learning method for reconstructing unknown functions from noisy measurements \cite{rao_consensus-based_2026}. As a nonparametric method, GPs do not assume a fixed functional form, making them flexible and advantageous in applications where dynamics are unknown or unpredictable \cite{balci_multiple_2025}. GPs are particularly attractive when measurements are limited or noisy due to their ability to quantify predictive uncertainty and incorporate prior knowledge \cite{lederer_cooperative_2023}. These properties make GPs well suited for multi-agent systems, which often rely on sensing, estimation, and uncertainty quantification \cite{rao_consensus-based_2026}. GPs have proven useful in a wide variety of multi-agent tasks, including wireless traffic prediction \cite{xu_wireless_2019}, safe multi-agent path planning \cite{zhu_multi-agent_2020, jakkala_multi-robot_2024}, and cooperative mapping \cite{ding_resource-efficient_2024}.

Applying GPs online in multi-agent systems nevertheless presents computational and communication challenges. Centralized approaches require agents to transmit their measurements to a common processing unit, creating a single point of failure, while standard GP regression scales cubically with the number of observations \cite{kontoudis_scalable_2024}.
Communication constraints present an additional challenge, which is particularly important in unmanned aerial vehicles (UAVs), multi-robot systems (MRSs), and wireless sensor networks (WSNs), where communication can constitute a significant fraction of the available energy budget \cite{gupta_survey_2015, gielis_critical_2022, razzaque_energy-efficient_2014}.

Prior works on distributed GPs have largely focused on kernel hyperparameter training e.g., \cite{xie_distributed_2019, zhai_distributed_2023} propose proximal-ADMM approaches for distributed hyperparameter training, while \cite{kontoudis_scalable_2024} introduce dataset augmentation to promote hyperparameter consensus. Although these methods enable distributed GP training, they do not address the cubic computational complexity of GPs or support online regression in distributed settings. The recursive GP (RGP) algorithm \cite{huber_recursive_2014} was proposed to resolve the complexity problem by enabling GP regression from streaming data using a sparse, inducing point-based approximation and Kalman filter-like updates. For multi-agent networks, Consensus-RGP \cite{rao_consensus-based_2026} extended RGP to model and predict multi-output functions using distributed consensus algorithms over a network. In a similar setting, D-RF-GP \cite{llorente_decentralized_2025} adopts random Fourier features and uses ensembles of models to adapt to changing kernel length scales. The more recent ROAD-GP \cite{llorente_robust_2026} further extends D-RF-GP to time-varying functions and improved its robustness to outliers.

Despite this progress, the communication and consensus efficiency of distributed online GP methods remain comparatively unexplored. Motivated by this gap, we develop two communication-efficient distributed recursive GP algorithms for multi-output regression that we call ADMM-RGP and PDMM-RGP. ADMM-RGP adapts a communication-efficient alternating direction method of multipliers (ADMM) scheme originally developed for distributed Kalman filtering \cite{iqbal_communication-efficient_2026}, while PDMM-RGP applies the primal-dual method of multipliers (PDMM) algorithm \cite{zhang_distributed_2018} to the distributed RGP problem. Compared with the state-of-the-art Consensus-RGP \cite{rao_consensus-based_2026}, we show that ADMM-RGP significantly reduces communication while maintaining comparable accuracy and network-wide consensus, whereas PDMM-RGP further reduces communication and is particularly effective in sparse communication networks. The main contributions of this work are as follows:
\begin{itemize}
  \item We develop ADMM-RGP and PDMM-RGP for distributed, recursive, multi-output GP regression.
  \item We analyze the stability of both algorithms and provide parameter selection strategies for fast convergence. With the proposed parameter selection, ADMM-RGP is shown to converge faster than the state-of-the-art.
  \item We evaluate the proposed methods on a real-world, multi-output wind dataset \cite{copernicus_climate_change_service_era5_2024} to demonstrate their effectiveness in learning an unknown function.
  \item We compare the communication and computational complexity of the proposed methods with the state of the art across graphs of varying connectivity.
\end{itemize} In this study, our key focus lies in the distributed fusion layer rather than in the local GP approximation. Consequently, ADMM-RGP and PDMM-RGP are not tied to a specific basis construction or kernel family. Building on the inducing point formulation of Consensus-RGP \cite{rao_consensus-based_2026}, we isolate the communication efficiency aspect of distributed inference, and thus enable a controlled comparison with the state-of-the-art methods. The proposed framework can be readily extended to alternative GP representations such as random Fourier features. Moreover, the inducing point formulation adopted here is kernel agnostic, and thus can accommodate a broader class of covariance functions.

\textit{Notation:} $\mathbf{I}_N \in \mathbb{R}^{N \times N}$ denotes the identity matrix, and $\mathbf{1}_N \in \mathbb{R}^N$ denotes a vector of ones. The $i$th entry of a vector $\mathbf{x}$ is denoted by $[\mathbf{x}]_i$, and the $(i,j)$th entry of a matrix $\mathbf{X}$ is denoted by $[\mathbf{X}]_{ij}$. Probability density functions are denoted by $p(\cdot)$, and the multivariate normal distribution with mean $\boldsymbol{\mu}$ and covariance $\boldsymbol{\Sigma}$ is denoted by $\mathcal{N}(\boldsymbol{\mu}, \boldsymbol{\Sigma})$. The Kronecker product is denoted by $\otimes$. The set of eigenvalues of a matrix $\mathbf{X}$ is denoted by $\lambda(\mathbf{X})$, while $\lambda_i$ denotes its $i$th smallest eigenvalue, such that $\lambda_1 \leq \cdots \leq \lambda_N$. Finally, $\|\cdot\|$ denotes the Euclidean norm for vectors and the spectral norm for matrices, while $\|\cdot\|_F$ denotes the Frobenius norm.

\section{Preliminaries} \label{sec:preliminaries}
\subsection{Multi-Output Gaussian Processes}

Consider a function $\mathbf{f}: \mathbb{R}^D \rightarrow \mathbb{R}^{D'}$that may be observed with some noise 
\begin{equation}
    \mathbf{y}_j = \mathbf{f}(\mathbf{x}_j)+ \mathbf{v}_j, \quad \mathbf{v}_j \sim \mathcal{N}(\mathbf{0}, \mathbf{R}_v) \label{eq:mo_measurement_model},
\end{equation}
where $\mathbf{y}_j$ is the noisy measurement at input location $\mathbf{x}_j$ and $\mathbf{R}_v$ is the measurement noise covariance matrix. With a collection of measurements we form a dataset $\mathcal{D} \triangleq \{ \mathbf{X}, \mathbf{y} \} $, where $\mathbf{X} \triangleq [\mathbf{x}_1 \ \cdots \ \mathbf{x}_J ]\in \mathbb{R}^{D\times J}$ is the matrix of measurement locations and $\mathbf{y} \triangleq [\mathbf{y}_1^\top \ \cdots \ \mathbf{y}_J^\top]^\top \in \mathbb{R}^{D'J}$ is the corresponding vector of measurements.

In particular, for $D' = 1$, the unknown function may be modeled using a single-output Gaussian process denoted as $\mathcal{GP}(m(\mathbf{x}), \kappa(\mathbf{x},\mathbf{x}'))$ which  is characterized by a mean function, $m(\mathbf{x})$, and a kernel function, $\kappa(\mathbf{x}, \mathbf{x}'):\mathbb{R}^D \times \mathbb{R}^D \rightarrow \mathbb{R}$, where the kernel captures the smoothness and range of values that the function is expected to exhibit. The kernel function can be used to construct a kernel matrix $[\mathbf{K}(\mathbf{X}, \mathbf{X}')]_{ij} \triangleq\kappa(\mathbf{x}_{i}, \mathbf{x}_{j}')$ for two sets of input locations ($\mathbf{X},\mathbf{X}')$.
The predicted distribution $\mathbf{f}(\mathbf{X}_*) \sim \mathcal{N}(\boldsymbol{\mu}_*, \boldsymbol{\Sigma}_*)$ at a set of input locations $\mathbf{X}_*$ is then 
\begin{subequations}
\begin{align}
    \boldsymbol{\mu}_* &= \mathbf{K}(\mathbf{X}_*,\mathbf{X})\mathbf{K}_\mathbf{y}^{-1}\mathbf{y},  \label{eq:pred1}\\ 
    \boldsymbol{\Sigma}_* &= \mathbf{K}(\mathbf{X}_*,\mathbf{X}_*) - \mathbf{K}(\mathbf{X}_*,\mathbf{X})\mathbf{K}_\mathbf{y}^{-1}\mathbf{K}(\mathbf{X},\mathbf{X}_*),\label{eq:pred3}
\end{align}
\end{subequations}
where $\mathbf{K}_\mathbf{y} \triangleq \mathbf{K}(\mathbf{X},\mathbf{X})+\mathbf{I}_J\otimes\mathbf{R}_v$. For multi-output functions in general, the linear model of coregionalization (LMC) may be used to capture correlations between different output dimensions \cite{alvarez_kernels_2012}. In the LMC approach, each output dimension is represented as a linear combination of $Q \le D'$ latent single-output functions, $\{u_q : \mathbb{R}^D \rightarrow \mathbb{R} \}_{q=1}^Q$. Each latent function has a zero-mean GP prior, $u_q(\mathbf{x}) \sim \mathcal{GP} (0, \kappa_q(\mathbf{x}, \mathbf{x}'))$, and a vector of weights $\mathbf{b}_q \triangleq [b_{1q} \ \cdots \ b_{D'q}]^\top$, where $\mathbf{b}_{iq}$ is associated with output dimension $i$. The covariance of the outputs at two input locations is given by the matrix-valued kernel function,
$\boldsymbol{\mathcal{K}}(\mathbf{x}, \mathbf{x}')\triangleq\text{cov}(\mathbf{f}(\mathbf{x}), \mathbf{f}(\mathbf{x}'))= \sum_{q=1}^Q \kappa_q(\mathbf{x}, \mathbf{x}') \mathbf{B}_q$, where $\mathbf{B}_q \triangleq \mathbf{b}_q \mathbf{b}_q^\top$ is the coregionalization matrix for latent function $q$. The LMC block kernel matrix for two sets of input locations ($\mathbf{X},\mathbf{X}')$ is thus defined as
$[\mathbf{K}(\mathbf{X}, \mathbf{X}')]_{ij} \triangleq
        \boldsymbol{\mathcal{K}}(\mathbf{x}_{i}, \mathbf{x}_{j}').$

\subsection{Recursive Gaussian Processes}

The recursive Gaussian process (RGP) algorithm provides an online estimate of latent function values $\mathbf{z} \triangleq \mathbf{f}(\mathbf{X}_p)$ at a fixed set of basis locations $\mathbf{X}_p \in \mathbb{R}^{D \times P}$ \cite{huber_recursive_2014}. At each time step $t$, a batch of $B_t<J$ noisy measurements $\{\mathbf{X}_t,\mathbf{y}_t\}$ is collected according to the measurement model in \eqref{eq:mo_measurement_model} and, using this batch and the latent prior $\mathbf{z} \mid \mathbf{y}_{1:t-1} \sim \mathcal{N}(\boldsymbol{\mu}_{t-1},\mathbf{\Sigma}_{t-1}),$ RGP obtains the posterior $\mathbf{z} \mid \mathbf{y}_{1:t} \sim \mathcal{N}(\boldsymbol{\mu}_t,\mathbf{\Sigma}_t)$, where $ \mathbf{y}_{1:t} \triangleq \{\mathbf{y}_i\}_{i=1}^{t}$. Given the prior distribution over the basis function values, the predictive distribution at the new measurement locations $\mathbf{f}_t \triangleq \mathbf{f}(\mathbf{X}_t) \sim \mathcal{N}(\boldsymbol{\mu}_{\mathbf{f}_t}, \mathbf{\Sigma}_{\mathbf{f}_t})$ is a Gaussian with
\begin{subequations}
\begin{align}
    \boldsymbol{\mu}_{\mathbf{f}_t}
    &= \mathbf{H}_t \boldsymbol{\mu}_{t-1}, \\
    \boldsymbol{\Sigma}_{\mathbf{f}_t}
    &= \mathbf{K}(\mathbf{X}_t,\mathbf{X}_t)
    - \mathbf{H}_t \mathbf{K}(\mathbf{X}_p,\mathbf{X}_t)
    + \mathbf{H}_t \mathbf{\Sigma}_{t-1} \mathbf{H}_t^\top, \\
    \mathbf{H}_t &\triangleq \mathbf{K}(\mathbf{X}_t,\mathbf{X}_p)\mathbf{K}_{\mathbf{X}_p}^{-1}, \label{eq:rgp_H_calculation} \\
    \mathbf{K}_{\mathbf{X}_p} &\triangleq \mathbf{K}(\mathbf{X}_p,\mathbf{X}_p) + \epsilon \mathbf{I},
\end{align}
\end{subequations}
where $\epsilon > 0$ is a small constant to ensure stability of the inversion. The posterior mean and covariance can be calculated using a Kalman-like correction update as
\begin{subequations}
\begin{align}
    \boldsymbol{\mu}_t &= \boldsymbol{\mu}_{t-1} + \mathbf{G}_t (\mathbf{y}_t-\mathbf{H}_t\boldsymbol{\mu}_{t-1}), \label{eq:rgp_mu_update} \\ 
    \mathbf{\Sigma}_t &= (\mathbf{I} - \mathbf{G}_t\mathbf{H}_t)\mathbf{\Sigma}_{t-1}, \label{eq:rgp_sigma_update} \\
    \mathbf{G}_t &\triangleq \mathbf{\Sigma}_{t-1}\mathbf{H}_t^\top  (\mathbf{R}_t+\mathbf{H}_t \mathbf{\Sigma}_{t-1} \mathbf{H}_t^{\top} )^{-1}, \label{eq:rgp_gain} \\
    \mathbf{R}_t &\triangleq \mathbf{K}(\mathbf{X}_t, \mathbf{X}_t) - \mathbf{H}_t \mathbf{K}(\mathbf{X}_p,\mathbf{X}_t) + \mathbf{I}_{B_t}\otimes\mathbf{R}_v \label{eq:rgp_R_calculation} .
\end{align}
\end{subequations}

\subsection{Distributed Recursive Gaussian Processes}

The distributed RGP problem considers a multi-agent network of $N \ge 2$ agents, whose communication topology is represented by a connected graph $\mathcal{G} = (\mathcal{V}, \mathcal{E})$, with vertices $\mathcal{V} = \{1, \cdots, N \}$ representing the agents and edges $\mathcal{E} \subseteq \mathcal{V} \times \mathcal{V}$ representing communication links. The set $\mathcal{N}_n \triangleq \{ m \in \mathcal{V} : (n,m) \in \mathcal{E} \}$ represents the neighbors of node $n$. At each time step $t$, each agent $n$ collects its own batch of $B_{n,t}$ independent noisy measurements $\{\mathbf{X}_{n,t},\mathbf{y}_{n,t}\}$ according to \eqref{eq:mo_measurement_model}. A set of $P$ basis locations, $\mathbf{X}_p$, are fixed and known to each agent.

The objective is for each agent to obtain the RGP estimate that would result from jointly processing all measurements collected across the network. Following the RGP formulation, the Gaussian distribution can be equivalently represented using the global information vector $\boldsymbol{\xi}_{g,t}$ and global information matrix $\mathbf{\Omega}_{g,t}$, given by 
\begin{align}
    \boldsymbol{\xi}_{g,t} &= \boldsymbol{\Sigma}_t^{-1}\boldsymbol{\mu}_t, \quad \mathbf{\Omega}_{g,t} = \boldsymbol{\Sigma}_t^{-1}.
\end{align}
Assuming the measurements at different agents are independent, the information-form global RGP updates are
\begin{align}
    \boldsymbol{\xi}_{g,t} &= \boldsymbol{\xi}_{g,t-1} + \sum_{n=1}^N \mathbf{H}_{n,t}^\top\mathbf{R}_{n,t}^{-1}\mathbf{y}_{n,t}, \label{eq:global_info_vector} \\
    \mathbf{\Omega}_{g,t} &= \mathbf{\Omega}_{g,t-1} + \sum_{n=1}^N \mathbf{H}_{n,t}^\top \mathbf{R}_{n,t}^{-1}\mathbf{H}_{n,t}, \label{eq:global_info_matrix}
\end{align}
where $\mathbf{H}_{n,t}$ and $\mathbf{R}_{n,t}$ are defined by \eqref{eq:rgp_H_calculation} and \eqref{eq:rgp_R_calculation} with the measurement batch collected by agent $n$. Although separable, \eqref{eq:global_info_vector}-\eqref{eq:global_info_matrix} require measurement contributions from all agents, which is not feasible in a distributed system with no central processing station. Each agent $n$ therefore maintains a local information vector $\boldsymbol{\xi}_{n,t}$ and an information matrix $\mathbf{\Omega}_{n,t}$ and exchanges information with neighboring agents to approximate the global solution.

\subsection{State of the Art: Consensus-RGP}

The Consensus-RGP algorithm extends the RGP framework to model multi-output functions in distributed, multi-agent systems \cite{rao_consensus-based_2026}. At each time step $t$, each agent $n$ applies updates \eqref{eq:rgp_mu_update} and \eqref{eq:rgp_sigma_update} in the information form with their local set of measurements
\begin{subequations}\label{eq:rgp_info_form_start}
    \begin{align}
        \boldsymbol{\xi}_{n,t} & = \boldsymbol{\xi}_{n,t-1} + \mathbf{H}_{n,t}^\top\mathbf{R}_{n,t}^{-1}\mathbf{y}_{n,t}, \\
        \mathbf{\Omega}_{n,t}  & = \mathbf{\Omega}_{n,t-1} + \mathbf{H}_{n,t}^\top \mathbf{R}_{n,t}^{-1}\mathbf{H}_{n,t},
    \end{align}
\end{subequations}
where
$\mathbf{H}_{n,t} \triangleq\mathbf{K}(\mathbf{X}_{n,t},\mathbf{X}_{p})\mathbf{K}_{\mathbf{X}_p}^{-1}$ and
$\mathbf{R}_{n,t}  \triangleq \mathbf{K}(\mathbf{X}_{n,t}, \mathbf{X}_{n,t})-\mathbf{H}_{n,t}\mathbf{K}(\mathbf{X}_p, \mathbf{X}_{n,t})+\mathbf{I}_{B_{n,t}} \otimes \mathbf{R}_v$. After the local updates, the agents perform $K$ rounds of distributed average consensus. In the $k^{th}$ round, agent $n$ updates its information parameters as
\begin{align}
    \boldsymbol{\xi}_{n,t}^{[k]} = \sum_{m=1}^Nw_{nm}\boldsymbol{\xi}_{m,t}^{[k-1]}, \quad
    \boldsymbol{\Omega}_{n,t}^{[k]} = \sum_{m=1}^Nw_{nm}\boldsymbol{\Omega}_{m,t}^{[k-1]},
\end{align}
where $w_{nm}$ are the row-stochastic consensus weights. After $K$ rounds of average consensus, the global information form parameters $\boldsymbol{\xi}_{g,t}$ and $\boldsymbol{\Omega}_{g,t}$ can be approximated as
\begin{subequations}\label{eq:rgp_info_form_end}
    \begin{align}
        \boldsymbol{\xi}_{g,t} & = N\boldsymbol{\xi}_{n,t}^{[K]},                                                               \\
        \mathbf{\Omega}_{g,t}  & =\mathbf{\Omega}_{n,0}^{[K]} + N(\boldsymbol{\Omega}_{n,t}^{[K]} - \boldsymbol{\Omega}_{n,0}).
    \end{align}
\end{subequations}

\section{Proposed Approaches} \label{sec:proposed_approaches}
\subsection{ADMM-RGP} \label{sec:admm_framework}

Inspired by a communication-efficient distributed ADMM variant of the Distributed Kalman Filter (DKF) \cite{iqbal_communication-efficient_2026}, we propose a novel distributed ADMM-RGP algorithm that presents the updates directly in information form, simplifying the convergence analysis. Let $\mathbf{L} \succeq \mathbf{0}$ be the Laplacian matrix of the communication graph. To derive distributed information vector updates using the ADMM-based approach of \cite{iqbal_communication-efficient_2026}, we formulate the separable average consensus optimization problem 
\begin{problem}
&\underset{\boldsymbol{\xi}_{t}}{\text{minimize}} \ \ \ \frac{1}{2}\sum_{n=1}^N \| \boldsymbol{\chi}_{n,t} - \boldsymbol{\xi}_{n,t} \|^2, \label{prob:info_vector_distributed} \\ &\text{subject to } \ \ \mathbf{L}_\xi \boldsymbol{\xi}_t = \mathbf{0}, \nonumber
\end{problem}
where $\boldsymbol{\xi}_t \triangleq[ \boldsymbol{\xi}_{1,t}^\top \ \cdots \ \boldsymbol{\xi}_{N,t}^\top]^\top$, $\boldsymbol{\chi}_{n,t} \triangleq N\mathbf{H}_{n,t}^\top \mathbf{R}_{n,t}^{-1}\mathbf{y}_{n,t} + \boldsymbol{\xi}_{n,t-1}$, $\mathbf{L}_\xi \triangleq (\mathbf{L} \otimes \mathbf{I}_{D_\xi})$, and $D_\xi \triangleq PD'$. The equality constraint ensures consensus among all agents, as the nullspace of $\mathbf{L}$ is the consensus subspace, spanned by $\mathbf{1}_N$. This problem has solution 
\begin{equation}
    \boldsymbol{\xi}_{t} ^*= \mathbf{1}_N\otimes\Big(\sum_{n=1}^N \mathbf{H}_{n,t}^\top \mathbf{R}_{n,t}^{-1}\mathbf{y}_{n,t}+ \frac{1}{N}\sum_{n=1}^N\boldsymbol{\xi}_{n,t-1}\Big). 
\end{equation}
If the agents reach consensus on the information vector at each time step, then $\boldsymbol{\xi}_{n,t} = \boldsymbol{\xi}_{g,t}$ for all $n \in \mathcal{V}$, as desired. To solve \ref{prob:info_vector_distributed}, we define the dual variable at node $n$, time $t$ as $\boldsymbol{\lambda}_{n,t} \in \mathbb{R}^{D_\xi}$ and let $\boldsymbol{\lambda}_t \triangleq [
    \boldsymbol{\lambda}_{1,t}^\top \ \cdots \ \boldsymbol{\lambda}_{N,t}^\top
]^\top$. The augmented Lagrangian with penalty parameter $\tau$ is then 
\begin{align}
\mathcal{L}_\tau(\boldsymbol{\xi}_t, \boldsymbol{\lambda}_t) = \frac{1}{2} \| \boldsymbol{\chi}_t - \boldsymbol{\xi}_{t} \|_2^2 + \boldsymbol{\lambda}_t^\top \mathbf{B}_\xi \boldsymbol{\xi}_t+\frac{\tau }{2} ||\mathbf{B}_\xi\boldsymbol{\xi}_t||_2^2, \label{eq:lagrangian1}
\end{align}
where $\mathbf{B}_\xi$ is the symmetric matrix square root of $\mathbf{L}_\xi$ and $\boldsymbol{\chi}_{t} \triangleq [\boldsymbol{\chi}_{1,t}^\top \ \cdots \ \boldsymbol{\chi}_{N,t}^\top]^\top$.
Differentiating \eqref{eq:lagrangian1} with respect to $\boldsymbol{\xi}_t$ and $\boldsymbol{\lambda}_t$ gives the following gradients.
\begin{subequations}
\begin{align}
    \nabla_{\boldsymbol{\xi}_t} \mathcal{L}_\tau &= \boldsymbol{\xi}_t -\boldsymbol{\chi}_t +
    \mathbf{B}_\xi\boldsymbol{\lambda}_t + \tau\mathbf{L}_\xi\boldsymbol{\xi}_t, \\
    \nabla_{\boldsymbol{\lambda}_t} \mathcal{L}_\tau &= \mathbf{B}_\xi\boldsymbol{\xi}_t.
\end{align}
\end{subequations}
Defining $\alpha >0$ as the step size of the dual variable ascent, the primal and dual variable updates at loop iteration $k$ are
\begin{subequations}
\begin{align}
    \tilde{\boldsymbol{\lambda}}_{t}^{[k]} &= \tilde{\boldsymbol{\lambda}}_{t}^{[k-1]} + \alpha\mathbf{L}_\xi\boldsymbol{\xi}_{t}^{[k-1]}, \\
    \boldsymbol{\xi}_{t}^{[k]}&=\boldsymbol{\chi}_t - \tilde{\boldsymbol{\lambda}}_{t}^{[k]}-\tau\mathbf{L}_\xi\boldsymbol{\xi}_{t}^{[k-1]}, \label{eq:mu_iterate}
\end{align}
\end{subequations}
where $\tilde{\boldsymbol{\lambda}}_t \triangleq\mathbf{B}_\xi\boldsymbol{\lambda}_t$ is an auxiliary dual variable. Let $l_{nm}\triangleq[\mathbf{L}]_{nm}$, then the local update at node $n$ is therefore
\begin{subequations}
\begin{align}
    \tilde{\boldsymbol{\lambda}}_{n,t}^{[k]}  &= \tilde{\boldsymbol{\lambda}}_{n,t}^{[k-1]}  + \alpha \sum_{m=1}^N l_{nm}\boldsymbol{\xi}_{m,t}^{[k-1]},  \label{eq:lambda_final}\\
    \boldsymbol{\xi}_{n,t}^{[k]}  &= \boldsymbol{\chi}_{n,t} - \tilde{\boldsymbol{\lambda}}_{n,t}^{[k]}  -\tau \sum_{m=1}^N l_{nm}\boldsymbol{\xi}_{m,t}^{[k-1]}.\label{eq:mu_final}
\end{align}
\end{subequations}

Now, to derive the distributed ADMM-based information matrix updates, we define the half-vectorization of the information matrix, $\boldsymbol{\omega} \triangleq \text{vech}(\boldsymbol{\Omega})$, which has dimension $D_\omega \triangleq \tfrac{1}{2}D_\xi(D_\xi+1)$. The distributed optimization problem is then
\begin{problem}
&\underset{\boldsymbol{\omega}_{t}}{\text{minimize}} \ \ \ \frac{1}{2}\sum_{n=1}^N \| \boldsymbol{\phi}_{n,t} -\boldsymbol{\omega}_{n,t} \|^2, \label{prob:covariance} \\ &\text{subject to } \ \ \mathbf{L}_\omega \boldsymbol{\omega}_t = \mathbf{0}, \nonumber 
\end{problem}
where $\boldsymbol{\omega}_t \triangleq [\boldsymbol{\omega}_{1,t}^\top \ \cdots \ \boldsymbol{\omega}_{N,t}^\top]^\top$, $\boldsymbol{\phi}_{n,t} \triangleq \text{vech}(N\mathbf{H}_{n,t}^\top \mathbf{R}_{n,t}^{-1} \mathbf{H}_{n,t}) + \boldsymbol{\omega}_{n,t-1}$, and $\mathbf{L}_\omega \triangleq \mathbf{L}\otimes\mathbf{I}_{D_\omega}$. Following a similar procedure as for the information vector, the information matrix updates are
\begin{subequations}
\begin{align}
    \tilde{\boldsymbol{\nu}}_t^{[k]} &= \tilde{\boldsymbol{\nu}}_t^{[k-1]} + \alpha\mathbf{L}_\omega\boldsymbol{\omega}_t^{[k-1]}, \\ 
    \boldsymbol{\omega}_t^{[k]} &=\boldsymbol{\phi}_t-\tilde{\boldsymbol{\nu}}_t^{[k]} - \tau \mathbf{L}_\omega \boldsymbol{\omega}_t^{[k-1]}, \label{eq:omega_iterate}
\end{align}
\end{subequations}
where $\tilde{\boldsymbol{\nu}}_{t} \in\mathbb{R}^{ND_\omega}$ is the global auxiliary dual variable. The fully distributed updates at node $n$ are
\begin{subequations}
\begin{align}
    \tilde{\boldsymbol{\nu}}_{n,t}^{[k]} &= \tilde{\boldsymbol{\nu}}_{n,t}^{[k-1]} + \alpha
    \sum_{m=1}^N l_{nm}\boldsymbol{\omega}_{m,t}^{[k-1]}, \label{eq:nu_final} \\
    \boldsymbol{\omega}_{n,t}^{[k]} &= \boldsymbol{\phi}_{n,t} - \tilde{\boldsymbol{\nu}}_{n,t}^{[k]}  -\tau \sum_{m=1}^N l_{nm}\boldsymbol{\omega}_{m,t}^{[k-1]} ,\label{eq:omega_final}
\end{align}
\end{subequations}
where $\tilde{\boldsymbol{\nu}}_{n,t} \in\mathbb{R}^{D_\omega}$ is the auxiliary dual variable at node $n$. After $K$ iterations of ADMM, the estimates of the mean $\boldsymbol{\mu}_{n,t}$ and covariance $\boldsymbol{\Sigma}_{n,t}$ can be obtained at each agent $n$ by computing

\begin{align} \label{eq:admm_final_outputs}
    \boldsymbol{\mu}_{n,t} = \mathbf{\Sigma}_{n,t}\boldsymbol{\xi}_{n,t}^{[K]}, \quad
    \mathbf{\Sigma}_{n,t} = \mathbf{\Omega}_{n,t}^{-1},
\end{align} where $\mathbf{\Omega}_{n,t} = \text{vech}^{-1}(\boldsymbol{\omega}_{n,t}^{[K]})$. The ADMM-RGP algorithm is summarized in Algorithm \ref{alg:admm_alg}.

\algnewcommand{\LineComment}[1]{\State \(\triangleright\)~#1}

\begin{algorithm}[t]
\caption{ADMM-RGP} 
\label{alg:admm_alg}
\begin{algorithmic}[1]
\State \textbf{Inputs:} $\mathbf{L}$, $\boldsymbol{\mathcal{K}}$, $\alpha$, $\tau$
\State $\boldsymbol{\omega}_{n,0} \gets \text{vech}((\mathbf{K}(\mathbf{X}_p, \mathbf{X}_p) + \epsilon \mathbf{I})^{-1})$ $\forall n \in \mathcal{V}$
\State $\boldsymbol{\xi}_{n,0} \gets \mathbf{0}$ $\forall n \in \mathcal{V}$

\For{$t = 1$ \textbf{to} $T$} 

\For{$n \in \mathcal{V}$} 
\State Get new batch of measurements $\{\mathbf{X}_{n,t}, \mathbf{y}_{n,t}\}$
\State Calculate $\boldsymbol{\chi}_{n,t}$ and $\boldsymbol{\phi}_{n,t}$
\State $\boldsymbol{\xi}_{n,t}^{[0]} \gets \boldsymbol{\chi}_{n,t}, \ \boldsymbol{\omega}_{n,t}^{[0]} \gets \boldsymbol{\phi}_{n,t}$
\State $\tilde{\boldsymbol{\lambda}}_{n,t}^{[0]} \gets\mathbf{0}, \tilde{\boldsymbol{\nu}}_{n,t}^{[0]} \gets\mathbf{0}$
\EndFor

\For{$k = 1$ \textbf{to} $K$} 
\For{$n \in \mathcal{V}$} 
    \State Share $\boldsymbol{\xi}_{n,t}^{[k-1]}, \boldsymbol{\omega}_{n,t}^{[k-1]}$ with all $m\in\mathcal{N}_n$
    \State Receive $\boldsymbol{\xi}_{m,t}^{[k-1]}, \boldsymbol{\omega}_{m,t}^{[k-1]}$ from all $m\in\mathcal{N}_n$
    \State Update $\tilde{\boldsymbol{\lambda}}_{n,t}^{[k]}$ and  $\boldsymbol{\xi}_{n,t}^{[k]}$ following \eqref{eq:lambda_final}--\eqref{eq:mu_final}
    \State Update $\tilde{\boldsymbol{\nu}}_{n,t}^{[k]}$ and $\boldsymbol{\omega}_{n,t}^{[k]}$ following \eqref{eq:nu_final}--\eqref{eq:omega_final}
\EndFor
\EndFor

\State $\boldsymbol{\xi}_{n,t} \gets \boldsymbol{\xi}_{n,t}^{[K]}$, $\boldsymbol{\omega}_{n,t} \gets \boldsymbol{\omega}_{n,t}^{[K]}$  for $\forall n \in \mathcal{V}$
\State Calculate $\boldsymbol{\mu}_{n,t}$ and $\mathbf{\Sigma}_{n,t}$ $\forall n \in \mathcal{V}$ using \eqref{eq:admm_final_outputs}
\EndFor
\State \textbf{Outputs:} $\boldsymbol{\mu}_{n,T}$ and $\mathbf{\Sigma}_{n,T}$ $\forall n \in \mathcal{V}$
\end{algorithmic}
\end{algorithm}
\normalsize

\subsection{PDMM-RGP} \label{sec:pdmm_framework}

The optimization problems in the ADMM-RGP derivation are node-separable with strongly convex, differentiable objectives and convex constraints, making them well suited for the primal–dual method of multipliers (PDMM) framework \cite{sherson_derivation_2019}. PDMM has been empirically shown to converge faster than ADMM in certain settings, including average consensus problems \cite{zhang_distributed_2018}.%

To reformulate \ref{prob:info_vector_distributed} for the application of PDMM, we define two directed edges $(n \rightarrow m)$ and $(m \rightarrow n)$ for each $(n,m) \in \mathcal{E}$. Furthermore let $\mathbf{A}_{n\rightarrow m} \triangleq a_{n\rightarrow m}\mathbf{I}_{D_\xi}$, where
\begin{align}
    a_{n\rightarrow m} \triangleq \begin{cases}
        -[\mathbf{L}]_{nm} & \text{if } n<m, \\
        [\mathbf{L}]_{nm} & \text{if } n>m,
    \end{cases} \label{eq:anm}
\end{align} which leads to the following reformulated problem
\begin{problem}
&\underset{\boldsymbol{\xi}_{t}}{\text{minimize}} \ \ \ \frac{1}{2}\sum_{n=1}^N \| \boldsymbol{\chi}_{n,t} - \boldsymbol{\xi}_{n,t} \|^2, \label{prob:info_vector_pdmm} 
\\ 
&\text{subject to } \ \ \mathbf{A}_{n \rightarrow m}\boldsymbol{\xi}_{n,t} + \mathbf{A}_{m \rightarrow n} \boldsymbol{\xi}_{m,t}= \mathbf{0} \quad \forall (n,m)\in \mathcal{E}. \nonumber
\end{problem}

Now, we construct a lifted constraint matrix, by assigning each directed edge $(n \rightarrow m)$ to a unique index $i_{n \rightarrow m} \in \{1,\dots,2M\}$, where $M = |\mathcal{E}|$ denotes the number of undirected edges in the communication graph. Let the lifted constraint matrix be denoted by $\mathbf{C} \in \mathbb{R}^{2M \times N}$ then \begin{equation}
    [\mathbf{C}]_{rn} = \begin{cases} a_{n \rightarrow m} & \text{if } r = i_{n \rightarrow m}, \\ 0 & \text{otherwise}.\end{cases}
\end{equation} Additionally, consider a $2M \times 2M$ permutation matrix $\mathbf{P}$ that exchanges the components at indices $i_{n\rightarrow m}$ and $i_{m\rightarrow n}$ for each undirected edge $(n,m)\in\mathcal{E}$.  Lastly, let the auxiliary dual variables be denoted by $\boldsymbol{\upsilon}_{n|m,t} \in \mathbb{R}^{D_\xi}$ at each directed edge $(n \rightarrow m)$ of the communication graph, which are aggregated into $\boldsymbol{\upsilon}_t \in \mathbb{R}^{2MD_\xi}$ following the same edge order used to build $\mathbf{C}$ and $\mathbf{P}$. 

The PDMM updates \cite{heusdens_distributed_2026} to solve \ref{prob:info_vector_pdmm} are then given by 
\begin{subequations}
\begin{align}
    \boldsymbol{\xi}_t^{[k]} &=  \arg \underset{\boldsymbol{\xi}}{\min} (\| \boldsymbol{\chi}_t - \boldsymbol{\xi}\| ^2+ (\mathbf{P}_\xi\boldsymbol{\upsilon}_t^{[k-1]})^\top\mathbf{C}_\xi\boldsymbol{\xi}+\frac{c}{2}\|\mathbf{C}_\xi\boldsymbol{\xi}\|^2), \nonumber \\
    \boldsymbol{\upsilon}_t^{[k]} &= \mathbf{P}_\xi\boldsymbol{\upsilon}_t^{[k-1]} + 2c\mathbf{C}_\xi\boldsymbol{\xi}_t^{[k]}, \nonumber
\end{align}
\end{subequations}
where $c > 0$ is a constant parameter,  $\mathbf{C}_\xi \triangleq \mathbf{C} \otimes \mathbf{I}_{D_\xi}$, and
$\mathbf{P}_\xi \triangleq \mathbf{P}\otimes\mathbf{I}_{D_\xi}$. To obtain the $\boldsymbol{\xi}_t$ update, take the gradient with respect to the primal variable and set it to zero. Because $\mathbf{C}^\top\mathbf{C}$ is a diagonal matrix with entries that sum the squared weights corresponding to each node and $c > 0$, $(\mathbf{I} +c\mathbf{C}_\xi^\top \mathbf{C}_\xi)$ is invertible, leading us to the global PDMM iterates for the information vector
\begin{subequations}
\begin{align}
    \boldsymbol{\xi}_{t}^{[k]} &= (\mathbf{I} + c \mathbf{C_\xi^\top C_\xi})^{-1}(\boldsymbol{\chi}_t - \mathbf{C}_\xi^\top \mathbf{P}_\xi\boldsymbol{\upsilon}_t^{[k-1]}), \\
    \boldsymbol{\upsilon}_{t}^{[k]} &= \mathbf{P}_\xi\boldsymbol{\upsilon}_{t}^{[k-1]} + 2c\mathbf{C}_\xi \boldsymbol{\xi}_{t}^{[k]}.
\end{align}
\end{subequations}
Subsequently, the distributed iterates at each agent $n$ are
\begin{subequations}
\begin{align}
    \boldsymbol{\xi}_{n,t}^{[k]} &=\frac{\boldsymbol{\chi}_{n,t} -\sum_{m\in\mathcal{N}_n} a_{n \rightarrow m} \boldsymbol{\upsilon}_{m|n,t}^{[k-1]}} {1 + c\sum_{m\in\mathcal{N}_n} a_{n \rightarrow m}^2}, \label{eq:xi_pdmm_update} \\
    \boldsymbol{\upsilon}_{n|m,t}^{[k]} &= \boldsymbol{\upsilon}_{m|n,t}^{[k-1]} + 2ca_{n \rightarrow m} \boldsymbol{\xi}_{n,t}^{[k]}. \label{eq:upsilon_pdmm_update} 
\end{align}
\end{subequations}

Note that the agent $n$ requires the value of $\boldsymbol{\upsilon}_{m|n,t}$ from agent $m$. Rather than transmitting these values individually over each edge using unicast communication, agent $m$ can broadcast its primal variable $\boldsymbol{\xi}_{m,t}$ to all of its neighbors, allowing each neighbor to locally reconstruct the value of $\boldsymbol{\upsilon}_{m|n,t}$. This broadcast strategy reduces the communication overhead.Along similar lines, \ref{prob:covariance} can be solved with PDMM by defining auxiliary dual variables $\boldsymbol{\psi}_{n|m,t} \in \mathbb{R}^{D_\omega}$ for each directed edge $(n \rightarrow m)$ in the communication graph. The local updates at node $n$ are
\begin{subequations}
\begin{align}
    \boldsymbol{\omega}_{n,t}^{[k]} &=\frac{\boldsymbol{\phi}_{n,t} -\sum_{m\in\mathcal{N}_n} a_{n \rightarrow m} \boldsymbol{\psi}_{m|n,t}^{[k-1]}} {1 + c\sum_{m\in\mathcal{N}_n} a_{n\rightarrow m}^2}, \label{eq:omega_pdmm_update} \\
    \boldsymbol{\psi}_{n|m,t}^{[k]} &= \boldsymbol{\psi}_{m|n,t}^{[k-1]} + 2ca_{n\rightarrow m} \boldsymbol{\omega}_{n,t}^{[k]}.
    \label{eq:psi_pdmm_update} 
\end{align}
\end{subequations}
The PDMM-RGP algorithm using a broadcast communication protocol is summarized in Algorithm \ref{alg:pdmm_alg}.

\begin{algorithm}[t]
\caption{PDMM-RGP} 
\label{alg:pdmm_alg}
\begin{algorithmic}[1]

\State \textbf{Inputs:} $\mathbf{L}$, $\boldsymbol{\mathcal{K}}$, $c$
\State $\boldsymbol{\omega}_{n,0} \gets \text{vech}((\mathbf{K}(\mathbf{X}_p, \mathbf{X}_p) + \epsilon \mathbf{I})^{-1})$, $\forall n \in \mathcal{V}$
\State $\boldsymbol{\xi}_{n,0} \gets \mathbf{0}$  $\forall n \in \mathcal{V}$

\For{$t = 1$ \textbf{to} $T$} 

\For{$n \in \mathcal{V}$} 
\State Get new batch of measurements $(\mathbf{X}_{n,t}, \mathbf{y}_{n,t})$
\State Calculate $\boldsymbol{\chi}_{n,t}$ and $\boldsymbol{\phi}_{n,t}$
\State $\boldsymbol{\xi}_{n,t}^{[0]} \gets 
\boldsymbol{\chi}_{n,t}, \ \boldsymbol{\omega}_{n,t}^{[0]} \gets \boldsymbol{\phi}_{n,t}$
\State $\boldsymbol{\upsilon}_{n|m,t}^{[0]} \gets\mathbf{0}, \boldsymbol{\psi}_{n|m,t}^{[0]} \gets\mathbf{0}$ for $m\in \mathcal{N}_n$ (at $n$)
\State $\boldsymbol{\upsilon}_{m|n,t}^{[0]} \gets\mathbf{0}, \boldsymbol{\psi}_{m|n,t}^{[0]} \gets\mathbf{0}$ for $m\in \mathcal{N}_n$ (at $n$)
\EndFor

\For{$k = 1$ \textbf{to} $K$} 
\For{$n \in \mathcal{V}$} 

    \State Update $\boldsymbol{\xi}_{n,t}^{[k]}$ and  $\boldsymbol{\upsilon}_{n|m,t}^{[k]}$ using \eqref{eq:xi_pdmm_update}-\eqref{eq:upsilon_pdmm_update}
    \State Update $\boldsymbol{\omega}_{n,t}^{[k]}$ and  $\boldsymbol{\psi}_{n|m,t}^{[k]}$  using \eqref{eq:omega_pdmm_update}-\eqref{eq:psi_pdmm_update}
    
    \State Share $\boldsymbol{\xi}_{n,t}^{[k]}, \boldsymbol{\omega}_{n,t}^{[k]}$ with all $m\in\mathcal{N}_n$
    \State Receive $\boldsymbol{\xi}_{m,t}^{[k]}, \boldsymbol{\omega}_{m,t}^{[k]}$ from all $m\in\mathcal{N}_n$

    \State $\boldsymbol{\upsilon}^{[k]}_{m|n,t} \gets \boldsymbol{\upsilon}_{n|m,t}^{[k-1]} + 2ca_{m \rightarrow n}\boldsymbol{\xi}_{m,t}^{[k]}$ (at $n$)
    \State $\boldsymbol{\psi}^{[k]}_{m|n,t} \gets \boldsymbol{\psi}_{n|m,t}^{[k-1]} + 2ca_{m \rightarrow n}\boldsymbol{\omega}_{m,t}^{[k]}$ (at $n$)
    
\EndFor
\EndFor

\State $\boldsymbol{\xi}_{n,t} \gets \boldsymbol{\xi}_{n,t}^{[K]}$, $\boldsymbol{\omega}_{n,t} \gets \boldsymbol{\omega}_{n,t}^{[K]}$  $\forall n \in \mathcal{V}$

\State Calculate $\boldsymbol{\mu}_{n,t}$ and $\mathbf{\Sigma}_{n,t}$ $\forall n \in \mathcal{V}$ using \eqref{eq:admm_final_outputs}
\EndFor

\State \textbf{Outputs:} $\boldsymbol{\mu}_{n,T}$ and $\mathbf{\Sigma}_{n,T}$ $\forall n \in \mathcal{V}$

\end{algorithmic}
\end{algorithm}
\normalsize

\section{Analysis}

In this section, we establish the convergence properties of the proposed ADMM-RGP and PDMM-RGP algorithms.

\subsection{ADMM-RGP Convergence}
We begin with the analysis in \cite{iqbal_communication-efficient_2026}. Recall the ADMM-RGP iterates from \eqref{eq:mu_iterate} and \eqref{eq:omega_iterate}, which can be written as a discrete-time, second-order system of the form
\begin{subequations}
    \begin{align}
        \boldsymbol{\xi}_{t}^{[k]}
         & = (\mathbf{I} - (\alpha + \tau)\mathbf{L}_\xi) \boldsymbol{\xi}_{t}^{[k-1]} + \tau \mathbf{L}_\xi \boldsymbol{\xi}_{t}^{[k-2]},             \\     \boldsymbol{\omega}_{t}^{[k]}
         & = (\mathbf{I} - (\alpha + \tau)\mathbf{L}_\omega) \boldsymbol{\omega}_{t}^{[k-1]} + \tau \mathbf{L}_\omega \boldsymbol{\omega}_{t}^{[k-2]}.
    \end{align}
\end{subequations}
We further define error terms
\begin{gather}
    \mathbf{e}_{t,\xi}^{[k]} \triangleq \boldsymbol{\xi}_t^* - \boldsymbol{\xi}_t^{[k]}, \quad \mathbf{e}_{t,\omega}^{[k]} \triangleq \boldsymbol{\omega}_t^* - \boldsymbol{\omega}_t^{[k]}, \\
    \mathbf{e}_{t}^{[k]} \triangleq \begin{bmatrix}
        \mathbf{e}_{t,\xi}^{[k]\top} & \mathbf{e}_{t,\omega}^{[k]\top}
    \end{bmatrix}^\top.
\end{gather}
Since the error $\mathbf{e}_t^{[k]}$ vanishes in the consensus subspace associated with $\lambda_1(\mathbf{L})=0$, we remove this component and diagonalize the remaining dynamics to obtain the transformed error $\tilde{\mathbf{e}}_t^{[k]} \in \mathbb{R}^{(N-1)(D_\xi+D_\omega)}$, with dynamics
\begin{equation}
    \begin{bmatrix}
        \tilde{\mathbf{e}}_{t}^{[k]} \\ \tilde{\mathbf{e}}_{t}^{[k-1]}
    \end{bmatrix} = \underbrace{\begin{bmatrix}
            (\mathbf{I} - (\alpha + \tau)\bar{\mathbf{\Lambda}}) & \tau \bar{\mathbf{\Lambda}} \\
            \mathbf{I}                                           & \mathbf{0}
        \end{bmatrix}}_{\mathbf{M}}
    \begin{bmatrix}
        \tilde{\mathbf{e}}_{t}^{[k-1]} \\ \tilde{\mathbf{e}}_{t}^{[k-2]}
    \end{bmatrix}, \label{eq:error_dynamics}
\end{equation} where $\bar{\boldsymbol{\Lambda}} \triangleq \text{diag}(\lambda_2(\mathbf{L}), \dots, \lambda_N(\mathbf{L})) \otimes \mathbf{I}_{D_\xi+D_\omega}$ \cite{iqbal_communication-efficient_2026}. Then, the eigenvalues of $\mathbf{M}$ are given by the eigenvalues of $\mathbf{M}_i$ for $i = 2, \cdots, N$, where
\begin{eqnarray}
    \mathbf{M}_i &\triangleq \begin{bmatrix}
        1 - (\alpha + \tau)\mathbf{\lambda}_i(\mathbf{L}) & \tau \mathbf{\lambda}_i(\mathbf{L}) \\ 1 & 0
    \end{bmatrix}, \\
    \lambda(\mathbf{M}_i) &= \frac{1}{2} \Big(\tilde{\alpha}_i \pm \sqrt{\tilde{\alpha}_i^2+4\tau\lambda_i(\mathbf{L})} \ \Big), \label{eq:quadratic_equation}
\end{eqnarray} where $\tilde{\alpha}_i \triangleq 1-(\alpha+\tau)\lambda_i(\mathbf{L})$ \cite{iqbal_communication-efficient_2026}. Since the transformed error dynamics constitute a discrete-time linear system, convergence of the iterates is guaranteed if $\mathbb{M}$ is Schur stable i.e., if all eigenvalues of $\mathbf{M}$ lie strictly inside the unit circle. In this case, $\mathbf{M}$ is shown to be Schur stable provided that  $\alpha + 2\tau < 2(\lambda_\text{N}(\mathbf{L}))^{-1}$ and $\alpha, \tau >0$ \cite{iqbal_communication-efficient_2026}. %

Note that the values of $\alpha$ and $\tau$ may be selected to minimize the spectral radius, $\rho(\mathbf{M})$, for fast convergence. Assuming positive $\alpha, \tau >0$, the eigenvalues of $\mathbf{M}$ are real, and the spectral radius can be expressed as
\begin{equation}
    \rho(\mathbf{M}) = \max_{i,j} |\lambda_j(\mathbf{M}_i)|
    = \max_{i} \  \frac{1}{2} ( |\tilde{\alpha}_i| + \sqrt{\tilde{\alpha}_i^2 + 4\tau \lambda_i(\mathbf{L})} ).\nonumber
\end{equation}
The $4\tau \lambda_i(\mathbf{L})$ term increases the maximum eigenvalue, meaning that $\tau$ should be chosen as close to $0$ as possible to minimize the spectral radius. However, when $\tau = 0$, the ADMM iterates are identical to first-order average consensus. Therefore, we observe that for $\tau > 0$, ADMM-RGP is guaranteed to be slower than Consensus-RGP. However, for $\tau < 0$, faster convergence could be achieved while preserving complexity. The augmented Lagrangian in \eqref{eq:lagrangian1} may be written as
\begin{align}
    \mathcal{L}_\tau(\boldsymbol{\xi}, \boldsymbol{\lambda}) & = \frac{1}{2}\boldsymbol{\xi}^\top (\mathbf{I}+\tau \mathbf{L}_\xi)\boldsymbol{\xi} -(\boldsymbol{\chi}^\top+ \boldsymbol{\lambda}^\top \mathbf{B}_\xi) \boldsymbol{\xi} +\frac{1}{2}\boldsymbol{\chi}^\top\boldsymbol{\chi},
    \nonumber
\end{align} which is convex in the primal variable if $(\mathbf{I} + \tau \mathbf{L}_\xi) \succeq \mathbf{0}$. For $\tau < 0$, the smallest eigenvalue of $(\mathbf{I} + \tau \mathbf{L}_\xi)$ is $1 + \tau \lambda_N(\mathbf{L})$. Thus, the convexity of the augmented Lagrangian is preserved for $\tau > -(\lambda_N(\mathbf{L}))^{-1}$, leading to Theorem \ref{theorem:admm_stability}.

\begin{theorem} \label{theorem:admm_stability}
    Consider the error dynamics matrix $\mathbf{M}$ defined in \eqref{eq:error_dynamics} and its eigenvalues given by \eqref{eq:quadratic_equation}. If $\alpha$ and $\tau$ satisfy
    \begin{equation}
        \alpha > 0, \quad \  \frac{-1}{\lambda_N(\mathbf{L})}<\tau<0, \quad \alpha
        +2\tau < \frac{2}{\lambda_N(\mathbf{L})},
    \end{equation} then $\mathbf{M}$ is Schur stable and ADMM-RGP converges.
\end{theorem}
\vspace{-0.5cm}
\begin{proof}
    See Appendix \ref{sec:theorem_proof}.
\end{proof}

The parameters $\alpha$ and $\tau$ may be selected using a grid search over the range of values given in Theorem \ref{theorem:admm_stability} to find the values that jointly minimize $\rho(\mathbf{M})$. To reduce communication costs, $K$ may be small enough that transient behavior dominates the error dynamics, in which case the parameters may be selected to minimize $\| \mathbf{M}^K\|_F$.

\subsection{PDMM-RGP Convergence} \label{sec:pdmm_convergence_and_params}

We begin with the observation that the objective functions of \ref{prob:info_vector_distributed} and \ref{prob:covariance} are closed, strongly convex, proper, and differentiable, so the PDMM iterations are guaranteed to converge for any $c > 0$ \cite{sherson_derivation_2019}. While $c$ does not affect asymptotic stability, note that it does influence convergence speed, a property of particular interest in communication-constrained systems where the number of $k$-iterations per time step may be limited. To analyze the convergence of PDMM-RGP, we define an auxiliary dual variable, $\boldsymbol{\zeta}_t^{[k]}\triangleq\mathbf{P}_\xi\boldsymbol{\upsilon}_t^{[k]}$, which has dynamics
\begin{subequations}
    \begin{align}
        \boldsymbol{\zeta}_t^{[k]}
         & = \mathbf{A}_\zeta\boldsymbol{\zeta}_t^{[k-1]} +2c\mathbf{P}_\xi\mathbf{C}_\xi \bar{\mathbf{C}}_c\boldsymbol{\chi}_t, \\ \bar{\mathbf{C}}_c &\triangleq (\mathbf{I}+c\mathbf{C_\xi^\top\mathbf{C}_\xi})^{-1}, \\ \mathbf{A}_\zeta &\triangleq \mathbf{P}_\xi - 2c\mathbf{P}_\xi\mathbf{C}_\xi \bar{\mathbf{C}}_c\mathbf{C}_\xi^\top.
    \end{align}
\end{subequations}
Furthermore, we define subspaces $\Psi \triangleq \text{ran}(\mathbf{C}_\xi) + \text{ran}(\mathbf{P}_\xi\mathbf{C}_\xi) \subseteq \mathbb{R}^{2MD_\xi}$ and $\Psi^\perp \triangleq \text{ker}(\mathbf{C}^\top) \cap \text{ker}((\mathbf{PC})^\top) \subseteq \mathbb{R}^{2MD_\xi}$. Let $\boldsymbol{\zeta}_{\Psi,t}^{[k]}$ and $\boldsymbol{\zeta}_{\Psi^\perp,t}^{[k]}$ denote the $\Psi$ and $\Psi^\perp$ components of $\boldsymbol{\zeta}_{t}^{[k]}$, respectively. %
According to Lemma 5.1 in \cite{li_communication_2022}, $\Psi$ and $\Psi^\perp$ are invariant over $\mathbf{P}_\xi$. In other words,
$
    \mathbf{x} \in \Psi \implies \mathbf{P}_\xi\mathbf{x} \in \Psi, \ \mathbf{y} \in \Psi^\perp \implies \mathbf{P}_\xi\mathbf{y} \in \Psi^\perp$.
Furthermore, $\Psi$ and $\Psi^\perp$ are invariant under $\mathbf{A}_\zeta$.
\begin{subequations}
    \begin{align}
        \mathbf{A}_\zeta \boldsymbol{\zeta}_{\Psi,t}^{[k]}       & = (\mathbf{P}_\xi - 2c\mathbf{P}_\xi\mathbf{C}_\xi \bar{\mathbf{C}}_c\mathbf{C}_\xi^\top)\boldsymbol{\zeta}^{[k]}_{\Psi,t} \in \Psi, \\
        \mathbf{A}_\zeta \boldsymbol{\zeta}_{\Psi^\perp,t}^{[k]} & = (\mathbf{P}_\xi - 2c\mathbf{P}_\xi\mathbf{C}_\xi \bar{\mathbf{C}}_c\mathbf{C}_\xi^\top)\boldsymbol{\zeta}^{[k]}_{\Psi^\perp,t}     \\
                                                                 & = \mathbf{P}_\xi\boldsymbol{\zeta}^{[k]}_{\Psi^\perp,t} \in \Psi^\perp.
    \end{align}
\end{subequations} Thus, the  $\boldsymbol{\zeta}_t^{[k]}$ iterations decompose into,
\begin{subequations}
    \begin{align}
        \boldsymbol{\zeta}_{\Psi,t}^{[k]}       & = \mathbf{A}_\Psi \boldsymbol{\zeta}_{\Psi,t}^{[k-1]} + \mathbf{\Pi}_\Psi(2c\mathbf{P}_\xi\mathbf{C}_\xi \bar{\mathbf{C}}_c\boldsymbol{\chi}_t),                       \\
        \boldsymbol{\zeta}_{\Psi^\perp,t}^{[k]} & = \mathbf{A}_{\Psi^\perp} \boldsymbol{\zeta}_{\Psi^\perp,t}^{[k-1]} + \mathbf{\Pi}_{\Psi^\perp}(2c\mathbf{P}_\xi\mathbf{C}_\xi \bar{\mathbf{C}}_c\boldsymbol{\chi}_t),
    \end{align}
\end{subequations}
where $\mathbf{A}_\Psi \triangleq \mathbf{\Pi}_\Psi \mathbf{A}_\zeta \mathbf{\Pi}_\Psi$ and $\mathbf{A}_{\Psi^\perp} \triangleq \mathbf{\Pi}_{\Psi^\perp} \mathbf{A}_\zeta \mathbf{\Pi}_{\Psi^\perp}$ are the restrictions of $\mathbf{A}_\zeta$ to the subspaces and $\mathbf{\Pi}_\Psi$ and $\mathbf{\Pi}_{\Psi^\perp}$ are the orthogonal projections onto $\Psi$ and $\Psi^\perp$.

\setcounter{figure}{1}
\begin{figure*}[b!]
    \setlength{\tabcolsep}{0pt}
    \centering
    \includegraphics[
        trim={1.1cm 3.5cm 2.95cm 0cm},
        clip,
        height=0.175\linewidth
    ]{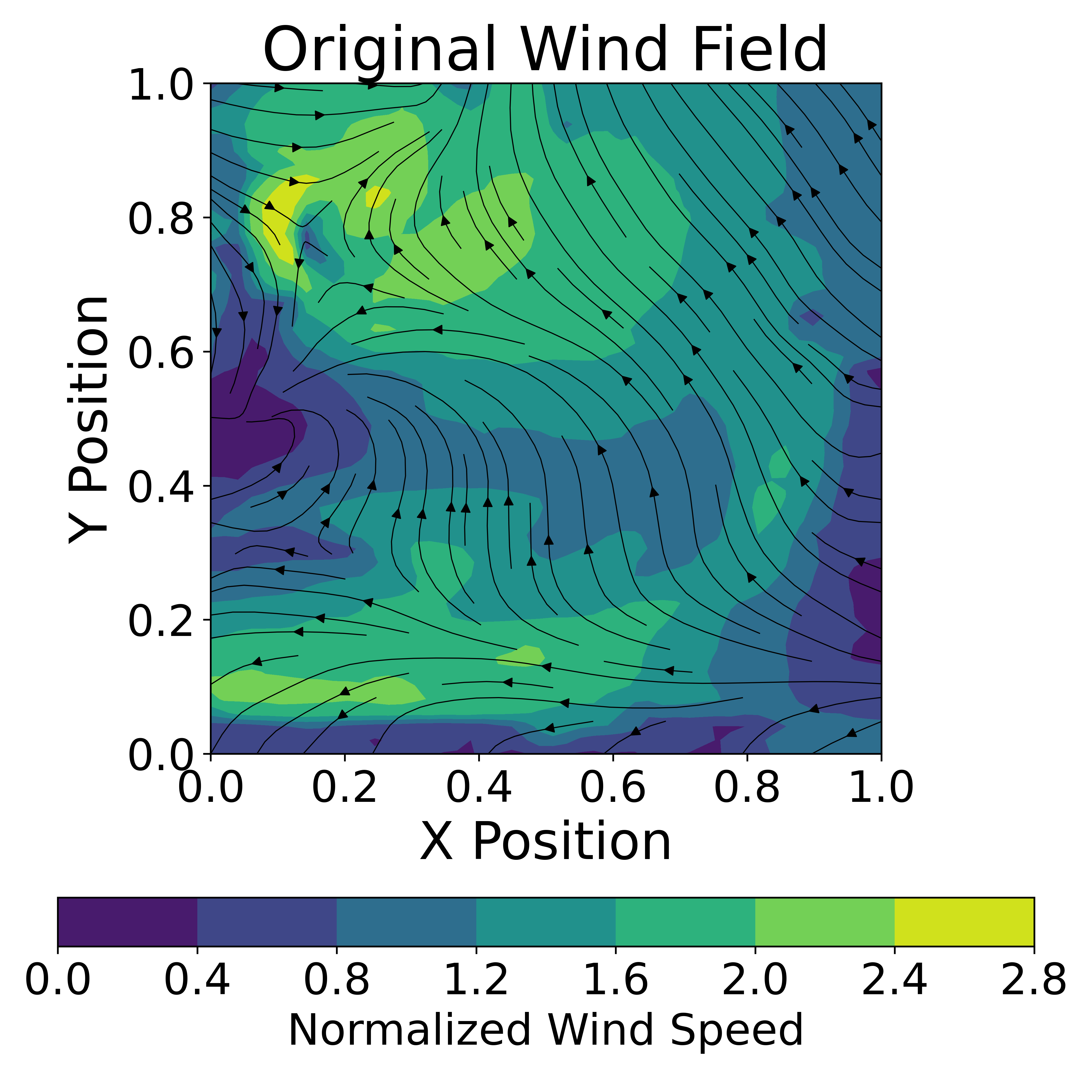}
    \includegraphics[
        trim={2cm 3.5cm 2.95cm 0cm},
        clip,
        height=0.175\linewidth
    ]{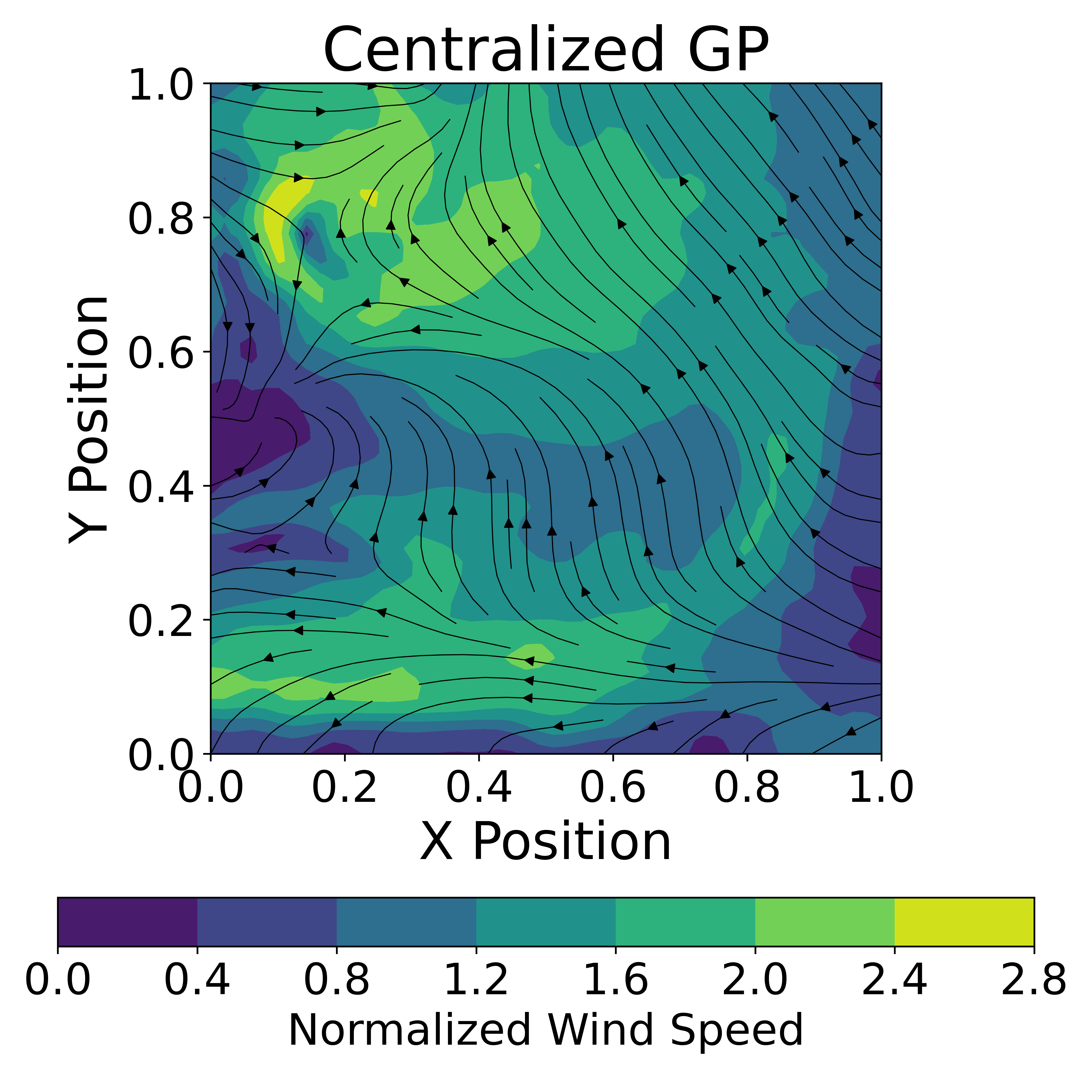}
    \includegraphics[
        trim={2cm 3.5cm 2.95cm 0cm},
        clip,
        height=0.175\linewidth
    ]{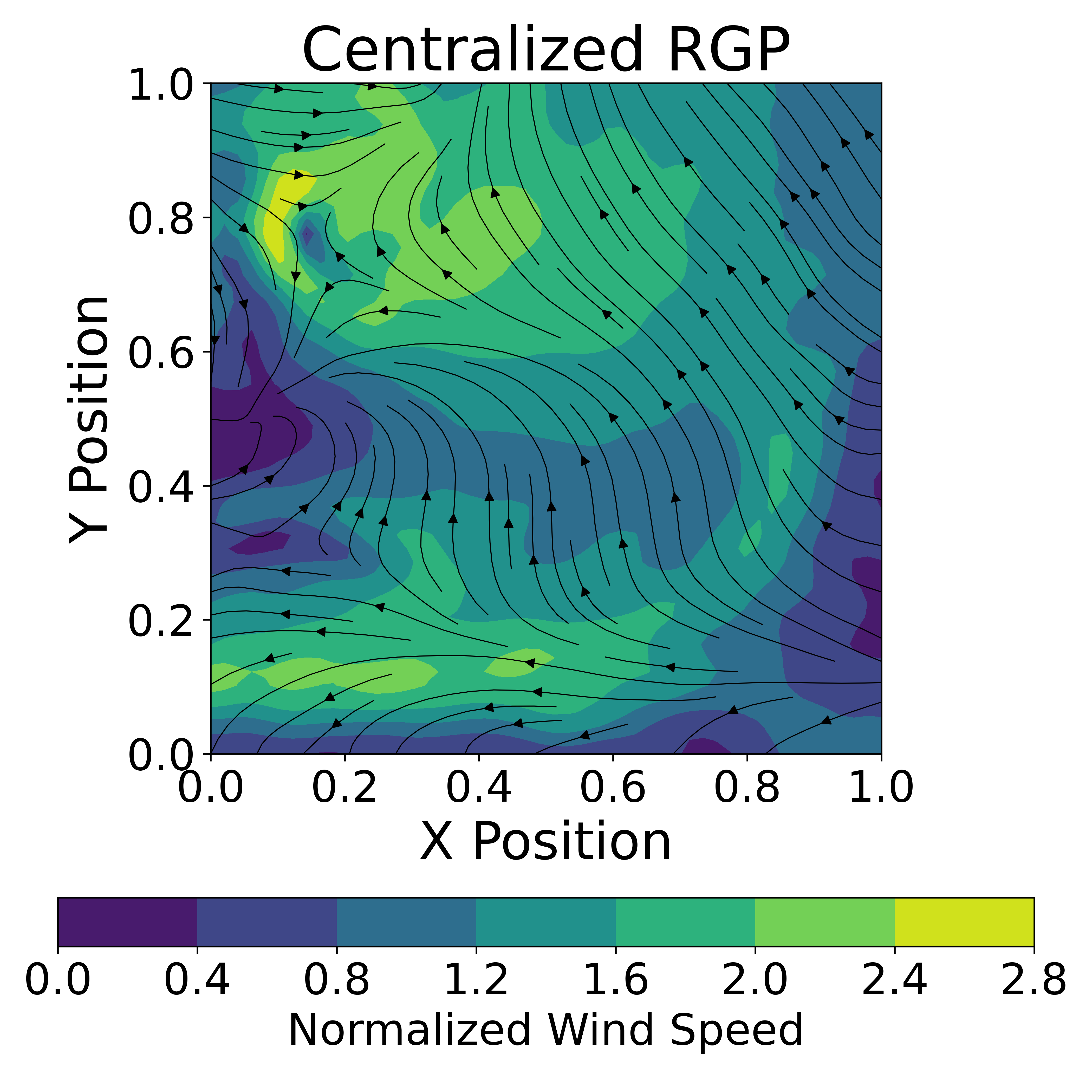}
    \includegraphics[
        trim={2cm 3.5cm 2.95cm 0cm},
        clip,
        height=0.175\linewidth
    ]{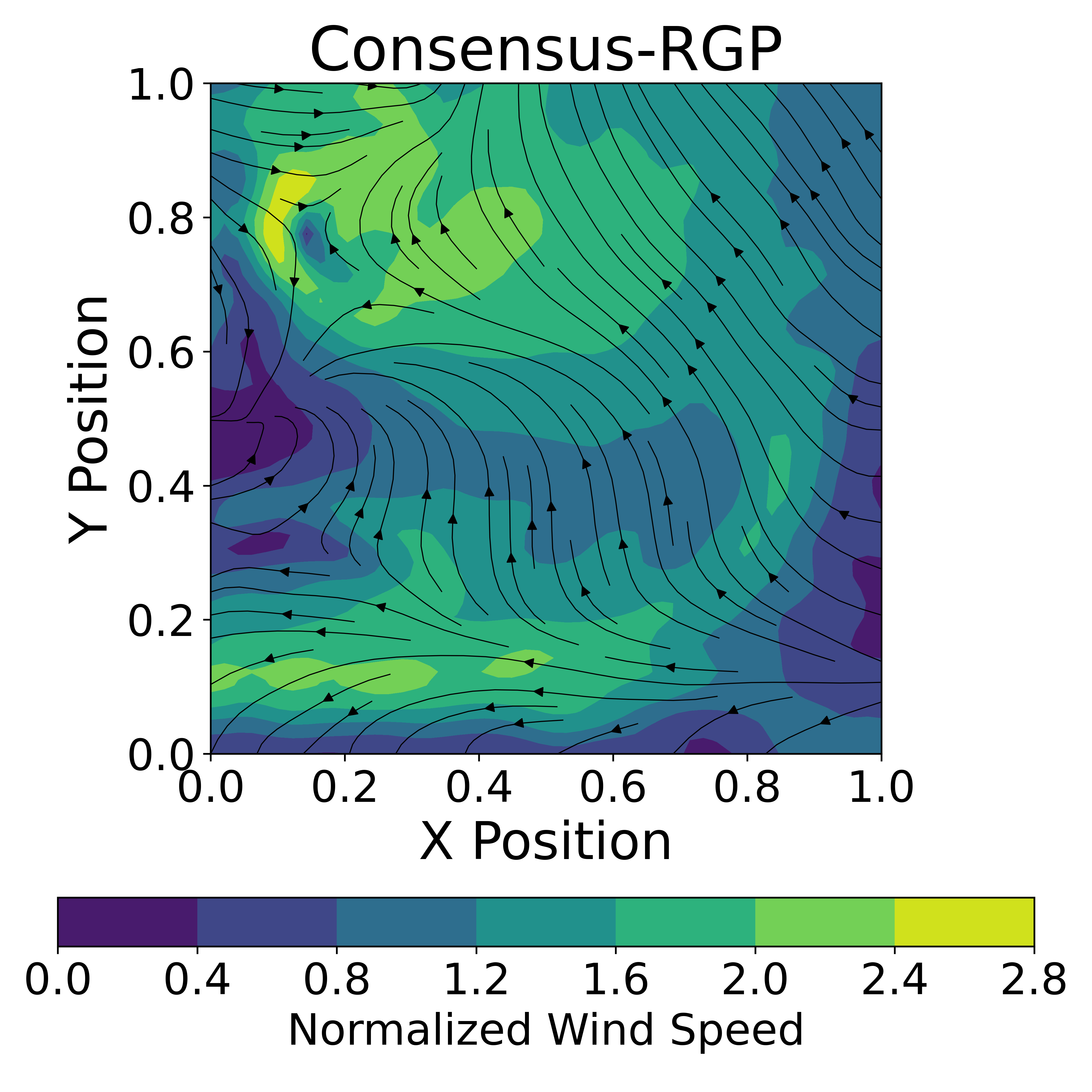}
    \includegraphics[
        trim={2cm 3.5cm 2.95cm 0cm},
        clip,
        height=0.175\linewidth
    ]{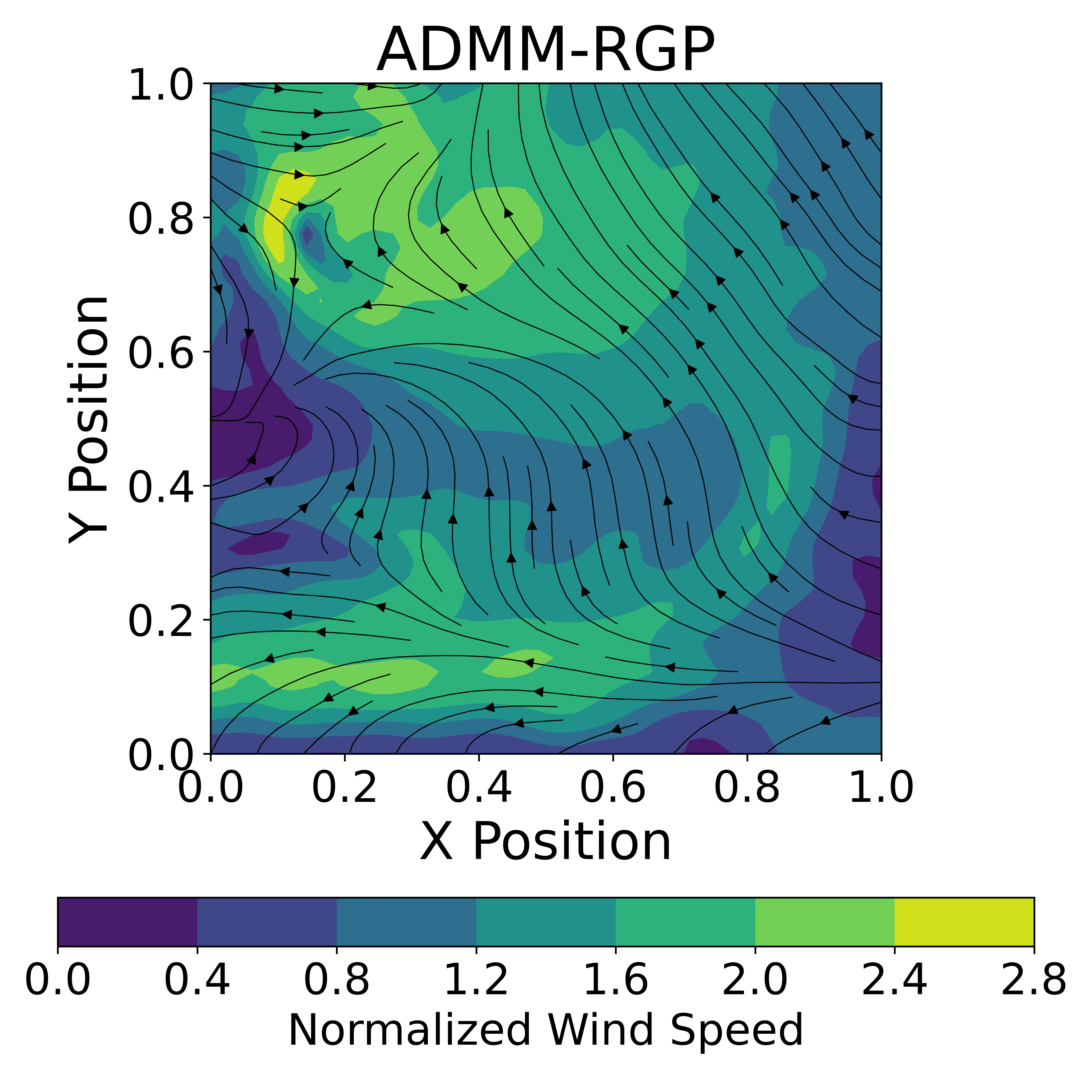}
    \includegraphics[
        trim={2cm 3.5cm 2.95cm 0cm},
        clip,
        height=0.175\linewidth
    ]{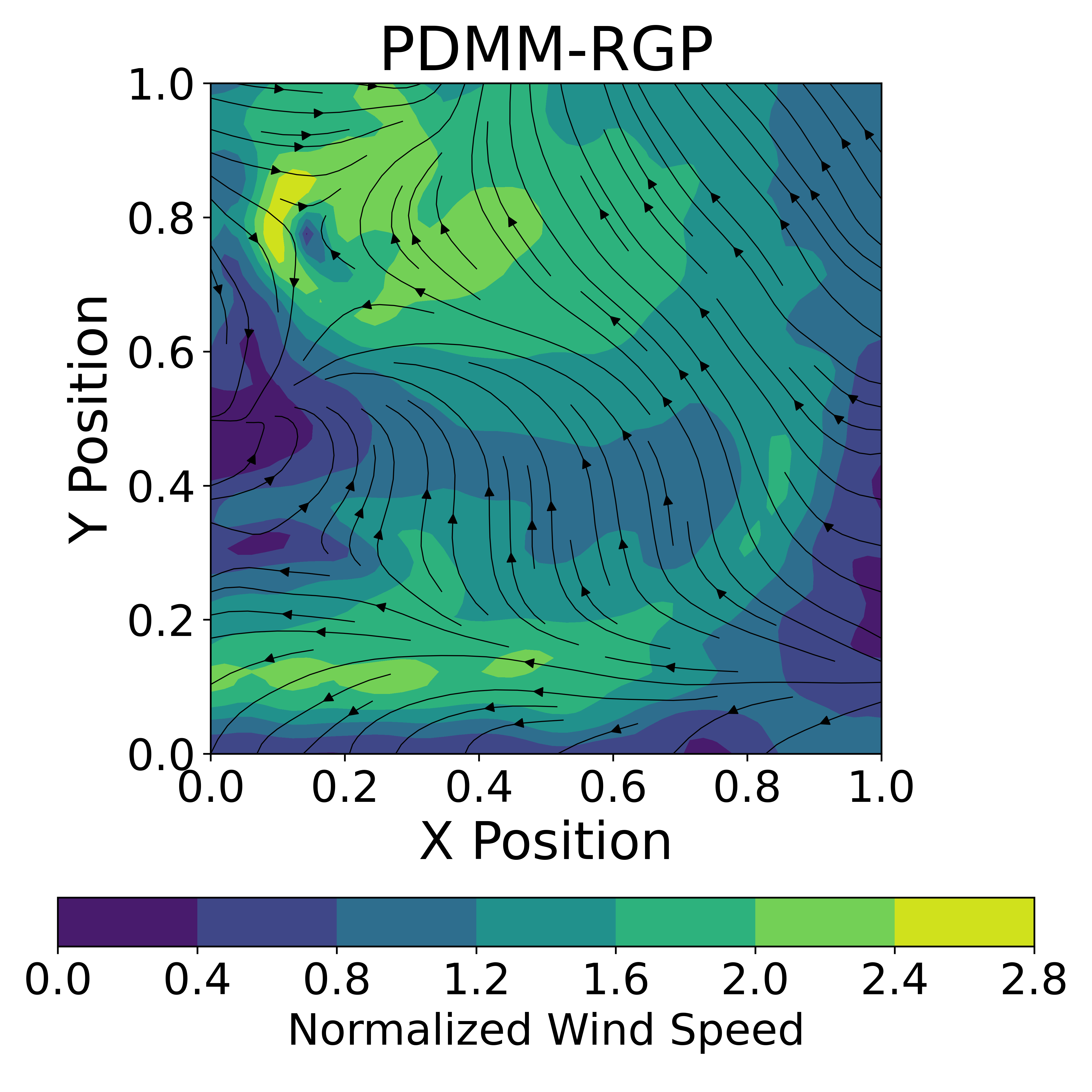} \\
    \begin{flushright}
        \includegraphics[trim={0cm 0.5cm -10cm 14.5cm}, clip, width=0.55\linewidth]{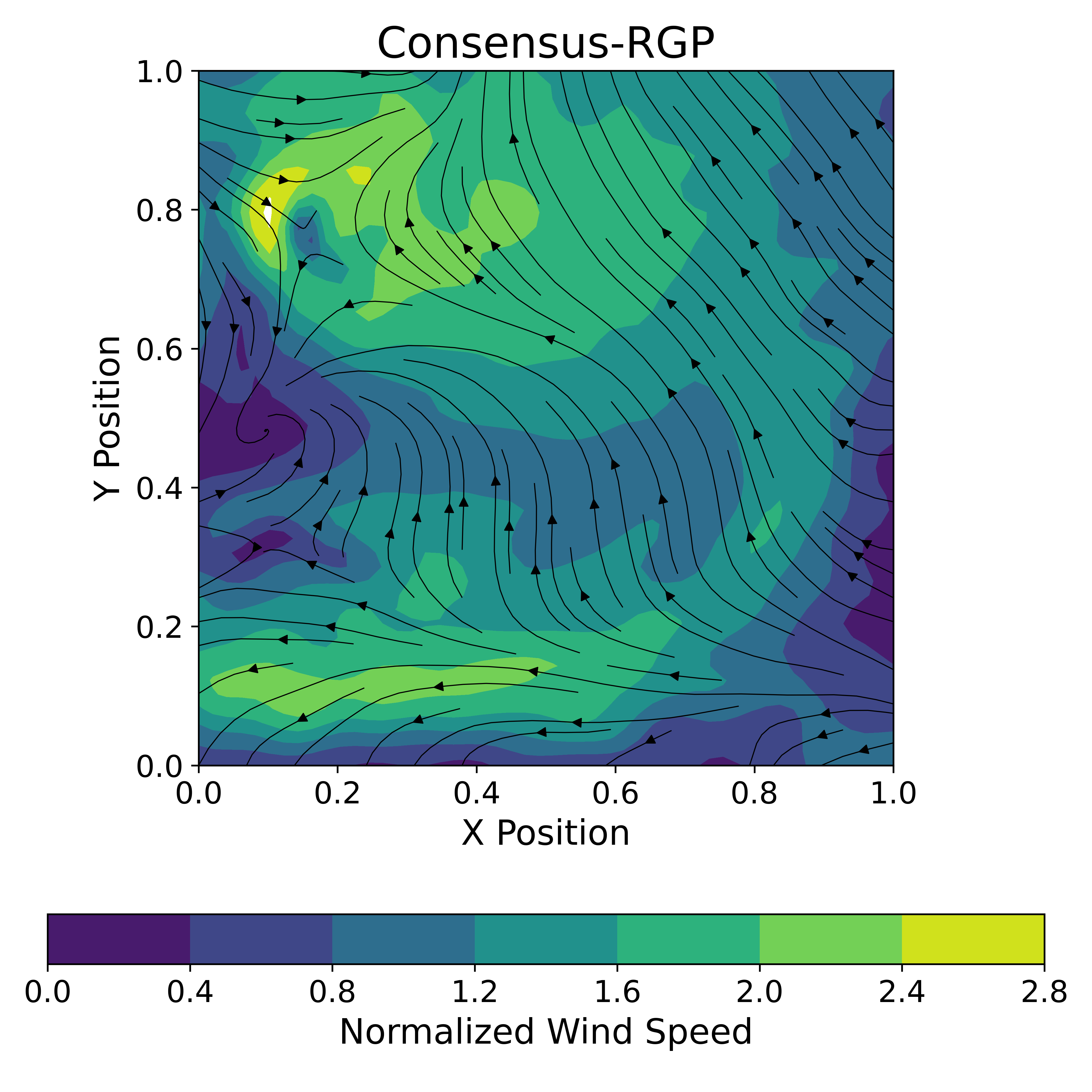}
    \end{flushright}
    \vspace{-1cm}
    \caption{Original and reconstructed wind fields.}
    \label{fig:reconstructed_fields}
\end{figure*} If $\boldsymbol{\zeta}_t^{[0]}$ is initialized randomly, it has been shown that only the component $\boldsymbol{\zeta}^{[k]}_{\Psi,t}$ converges to a fixed point $\boldsymbol{\zeta}^*_{\Psi,t}$ as $k \rightarrow \infty$ \cite{li_communication_2022}. Thus, define the error of the $\boldsymbol{\zeta}_t^{[k]}$ iterations as $\mathbf{e}_{\zeta,t}^{[k]} \triangleq \boldsymbol{\zeta}^*_{\Psi,t} - \boldsymbol{\zeta}_{\Psi,t}^{[k]}
    = \mathbf{A}_\Psi (\boldsymbol{\zeta}_{\Psi,t}^*-\boldsymbol{\zeta}_{\Psi,t}^{[k-1]})
    = \mathbf{A}_\Psi^k\mathbf{e}_{\zeta,t}^{[0]},$. Thus the $\boldsymbol{\zeta}$-iterations is governed by the spectral radius of $\mathbf{A}_\Psi$. The non-convergence of the $\Psi^\perp$ component does not affect the convergence of $\boldsymbol{\xi}_t^{[k]}$, as it is annihilated by $\mathbf{C}_\xi^\top$. Therefore, $\boldsymbol{\xi}_t^{[k]}$ will converge to a fixed point, $\boldsymbol{\xi}_t^{*}$. Now, observe the error of the $\boldsymbol{\xi}_t^{[k]}$ iterate is given by
\begin{subequations}
    \begin{align}
        \mathbf{e}_{\xi,t}^{[k]} & \triangleq \boldsymbol{\xi}_t^{*}-\boldsymbol{\xi}_t^{[k]} =-\bar{\mathbf{C}}_c\mathbf{C}_\xi^\top \boldsymbol{\zeta}_{\Psi,t}^* +\bar{\mathbf{C}}_c\mathbf{C}_\xi^\top \boldsymbol{\zeta}_{\Psi,t}^{[k-1]} \\
                                 & = -\bar{\mathbf{C}}_c\mathbf{C}_\xi^\top\mathbf{A}_\Psi^{k-1}\mathbf{e}_{\zeta,t}^{[0]}.
    \end{align}
\end{subequations} Observe that the $\boldsymbol{\xi}_t^{[k]}$ error is thus governed by the $\boldsymbol{\zeta}_{t}^{[k]}$ error, which converges asymptotically to $\mathbf{0}$ for any $c > 0$. To maximize the convergence rate, $c$ may be chosen to minimize $\rho(\mathbf{A}_\Psi)$; since this minimization is non-convex in $c$, standard convex optimization methods cannot be used, and $c$ must instead be selected via a search over candidates. While the spectral radius governs asymptotic convergence, transient behavior may still dominate for small $K$, in which case $c$ may instead be selected to minimize $\|\mathbf{C}_\xi^\top\mathbf{A}_\Psi^{K-1}\|_F$.

\setcounter{figure}{0}
\begin{figure}[t!]
    \centering
    \includegraphics[width=0.4\linewidth]{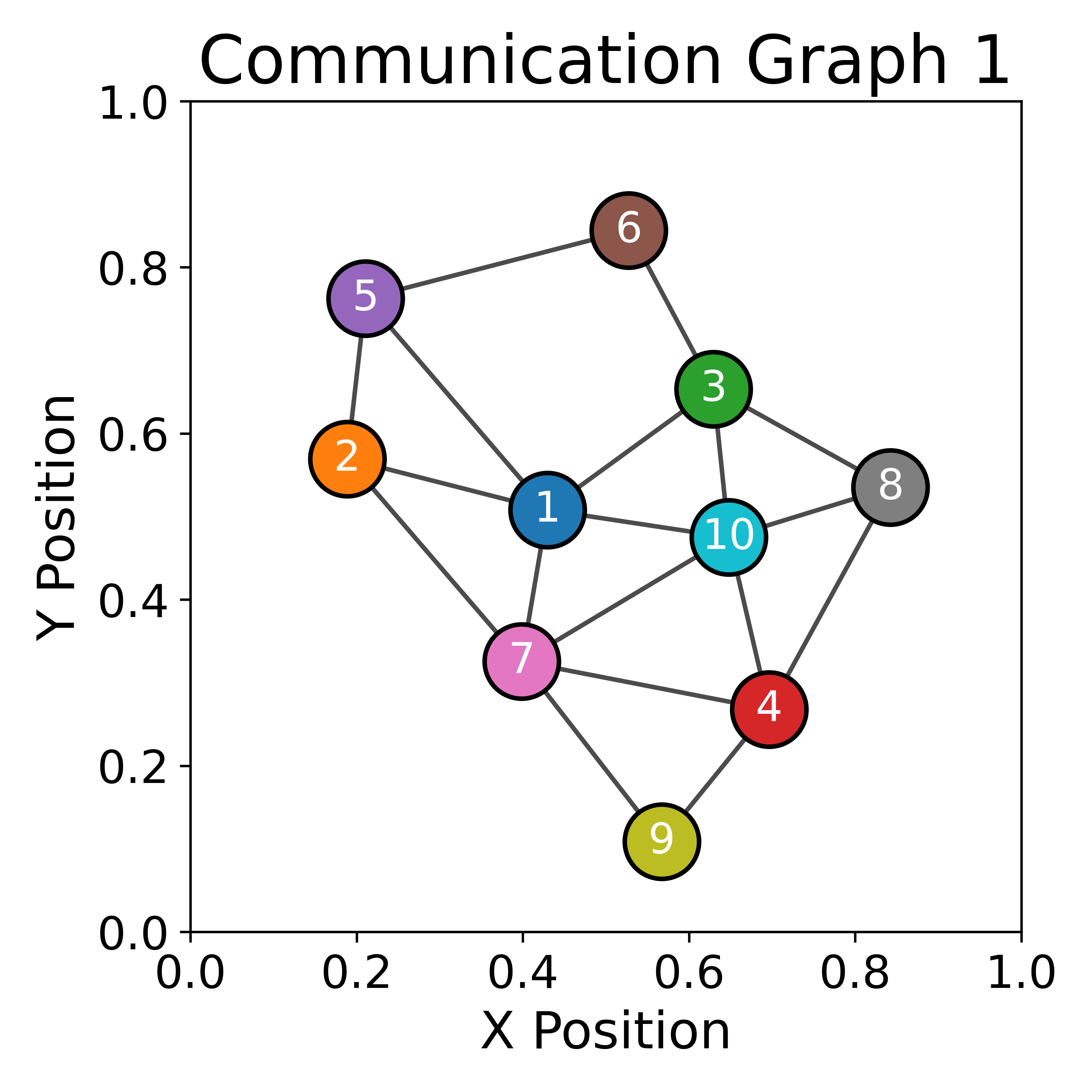}    \includegraphics[width=0.4\linewidth]{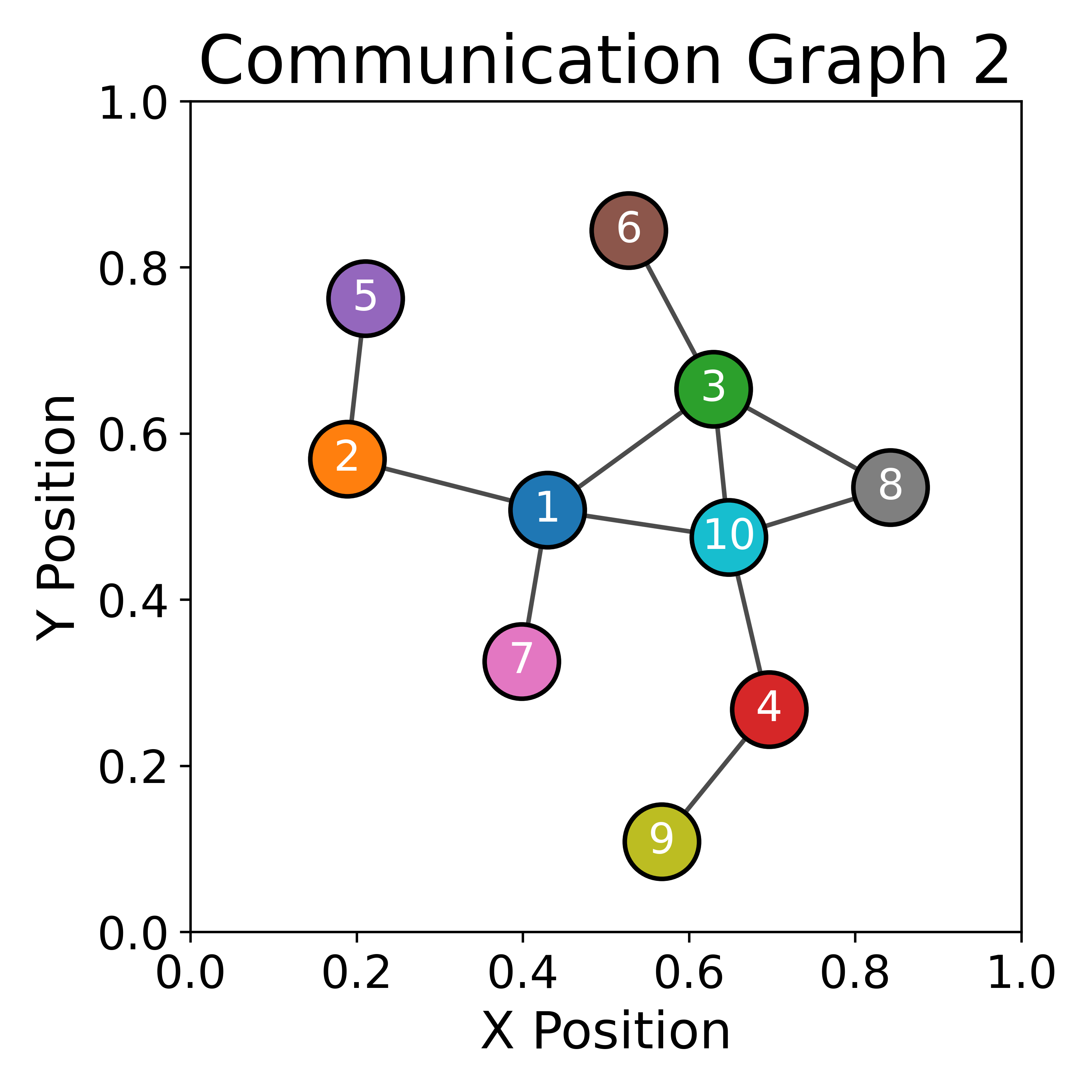}
    \caption{Communication graphs used in simulations.}
    \label{fig:comm_graphs}
    \vspace{-0.3cm}
\end{figure}

\section{Experiments} \label{sec:results}

The dataset used in the experiments consists of 10-meter wind $u$- and $v$-components indexed by GPS location, obtained from the ``ERA5 post-processed daily statistics on single levels from 1940 to present'' dataset provided by the Copernicus Climate Data Store \cite{copernicus_climate_change_service_era5_2024}.
Simulations use $N=10$ agents, whose fixed positions in the 2D wind field are shown in Figure \ref{fig:comm_graphs} with two different communication graphs. At each time step, each agent $n$ draws sample locations from $\mathcal{N}(\mathbf{p}_n,\sigma_s^2 \mathbf{I})$, where $\mathbf{p}_n$ is the agent's position and $\sigma_s = 0.25$. Given a measurement location $\mathbf{x}$, the measurement is drawn from $\mathcal{N}(\mathbf{f}(\mathbf{x}), \mathbf{R}_v)$, where $\mathbf{R}_v=\text{diag}(0.01, 0.01)$. Each agent collects $20$ measurements per time step for $T=50$ time steps, giving $10,000$ training data points. The set of 400 basis locations $\mathbf{X}_p$ form a 20-by-20 grid over the input space, and $Q=2$ latent functions are used, with the first function's kernel hyperparameters pre-trained on the $u$ wind speeds and the second trained on the $v$ wind speeds. 2,500 prediction locations $\mathbf{X}_*$, which form a 50-by-50 grid over the input space, are used to evaluate the trained model. Two Laplacian weightings are considered: the unweighted graph Laplacian and the symmetric ``optimal'' weights obtained by solving the spectral-norm minimization problem in \cite{xiao_fast_2004}. The parameters $\alpha$, $\tau$, and $c$ are tuned to optimize convergence speed for each algorithm.

\subsection{Performance Metrics}

Given the posterior covariances $\boldsymbol{\Sigma}_{n,t}= \boldsymbol{\Omega}_{n,t}^{-1}$ and means $\boldsymbol{\mu}_{n,t} = \boldsymbol{\Sigma}_{n,t} \boldsymbol{\xi}_{n,t}$ of agent $n$ at time $t$ at the basis locations, the predictive distribution at a set of test locations $\mathbf{X}_* \in \mathbb{R}^{D \times P_*}$ is $\mathcal{N}(\boldsymbol{\mu}_{*n,t}, \boldsymbol{\Sigma}_{*n,t})$, where $\boldsymbol{\mu}_{*n,t} = \mathbf{K}(\mathbf{X}_*, \mathbf{X}_p)\mathbf{K}_{\mathbf{X}_p}^{-1} \boldsymbol{\mu}_{n,t}$, and $ \mathbf{\Sigma}_{*n,t}= \mathbf{K}(\mathbf{X}_*, \mathbf{X}_*) + \mathbf{K}(\mathbf{X}_*, \mathbf{X}_p)(\mathbf{K}_{\mathbf{X}_p}^{-1} \boldsymbol{\Sigma}_{n,t} -\mathbf{I})\mathbf{K}_{\mathbf{X}_p}^{-1}\mathbf{K}(\mathbf{X}_p, \mathbf{X}_*)$. Now, letting $\mathbf{f}_* \in \mathbb{R}^{P_* D'}$ denote the true function values at the test locations, we use the root mean square error (RMSE) defined as
\begin{equation}
    \text{RMSE}(\mathbf{f}_*, \boldsymbol{\mu}_{*n,t}) \triangleq \sqrt{\dfrac{(\boldsymbol{\mu}_{*n,t}-\mathbf{f}_*)^\top (\boldsymbol{\mu}_{*n,t}-\mathbf{f}_*)}{P_*D'}},
\end{equation}
to evaluate the accuracy of the predictions.
To evaluate the level of consensus among agents, we use mean variance of predictions (MVOP) defined as
\begin{align}
    \text{MVOP}(\boldsymbol{\mu}_{*1,t}, \dots \boldsymbol{\mu}_{*N, t})\triangleq \frac{1}{P_*D'}\sum\limits_{i=1}^{P_*D'} \sigma^2_{i,t},
\end{align}
where $\sigma_{i,t}^2 \triangleq \frac{1}{N-1}\sum_{n=1}^N ([\boldsymbol{\mu}_{*n,t}]_i - \bar{\mu}_{i,t})^2$ is the unbiased sample variance of the $i^{th}$ prediction across $N$ agents at time $t$ and $\bar{\mu}_{i, t} = \frac{1}{N}\sum_{n=1}^N [\boldsymbol{\mu}_{*n,t}]_i$ is the sample mean.

\setcounter{figure}{2}
\begin{figure}[t]
    \centering
    \includegraphics[trim={0 3.05cm 0 0}, clip,width=0.99\linewidth]{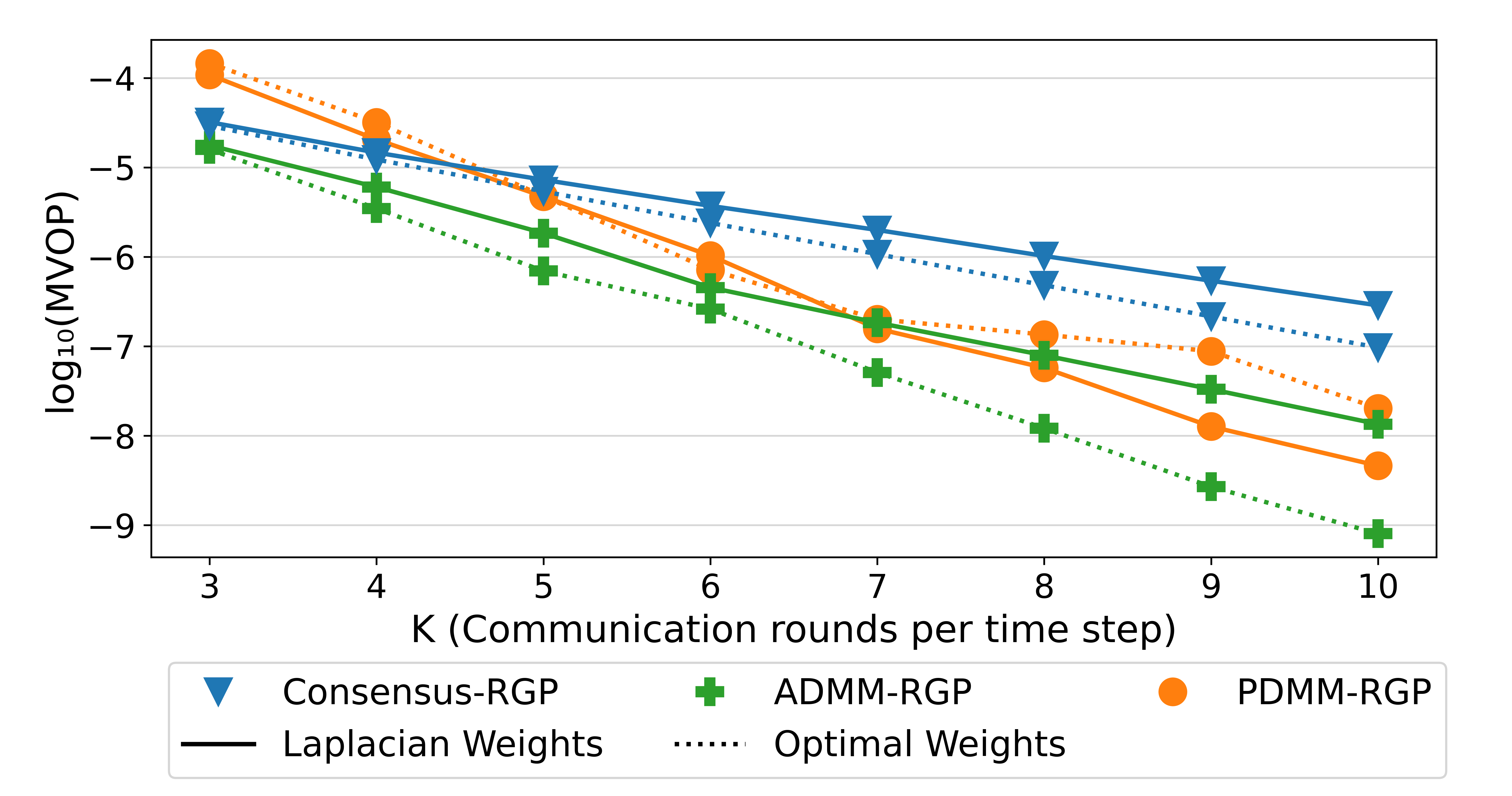}\vspace{0.1cm}
    \begin{overpic}[width=0.99\linewidth]{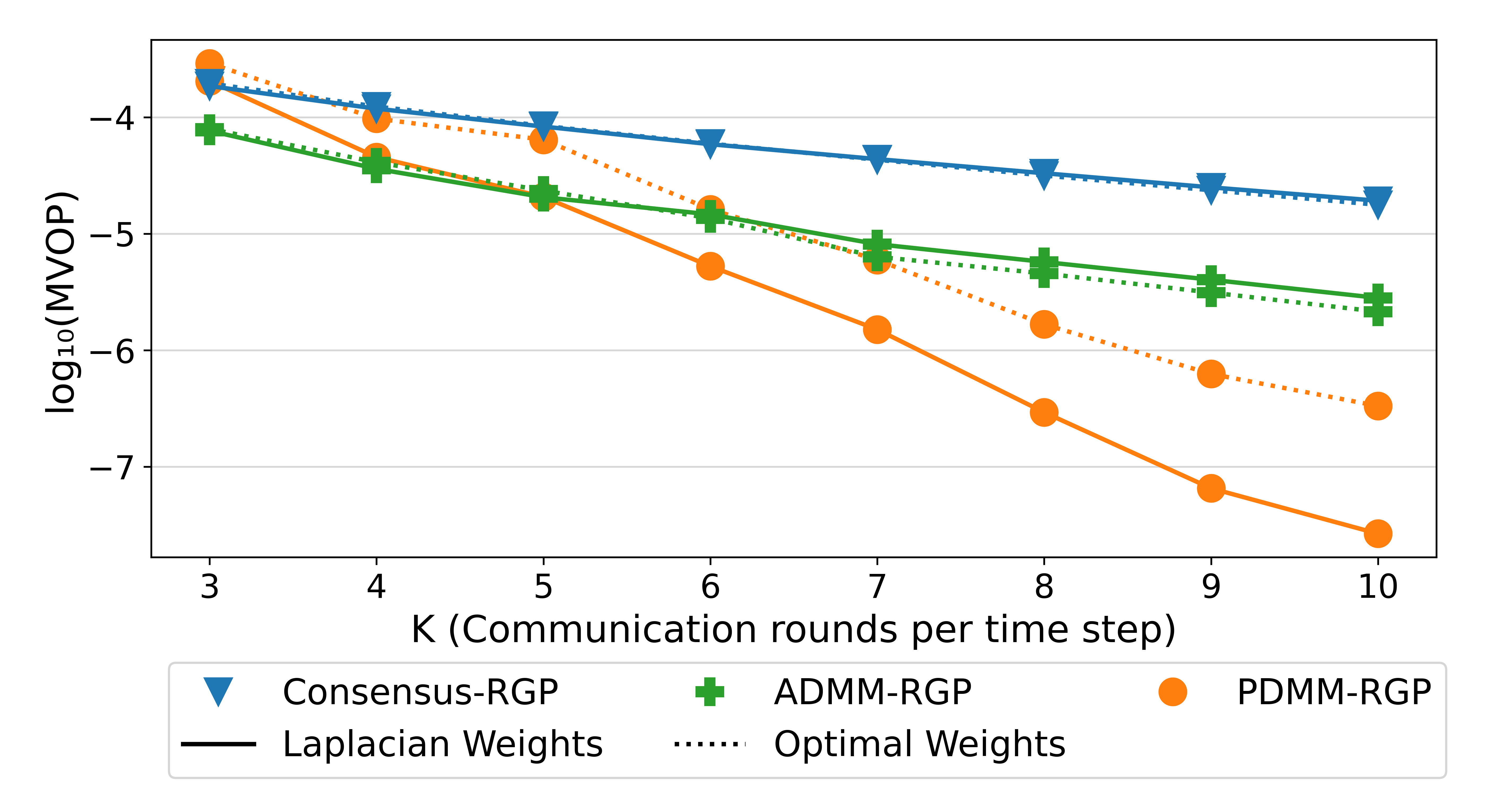}\put(37.8,6.33){\fontfamily{DejaVuSans-TLF}\fontsize{5.2pt}{7pt}\selectfont{\cite{rao_consensus-based_2026}}}
    \end{overpic}
    \caption{Plots of consensus metric MVOP vs. number of communication rounds $K$ using communication graph 1 (top) and 2 (bottom) and $T=20$ time steps.}
    \label{fig:ksweep}
    \vspace{-0.2cm}
\end{figure}

\begin{figure}[t]
    \centering
    \begin{overpic}[width=0.99\linewidth]{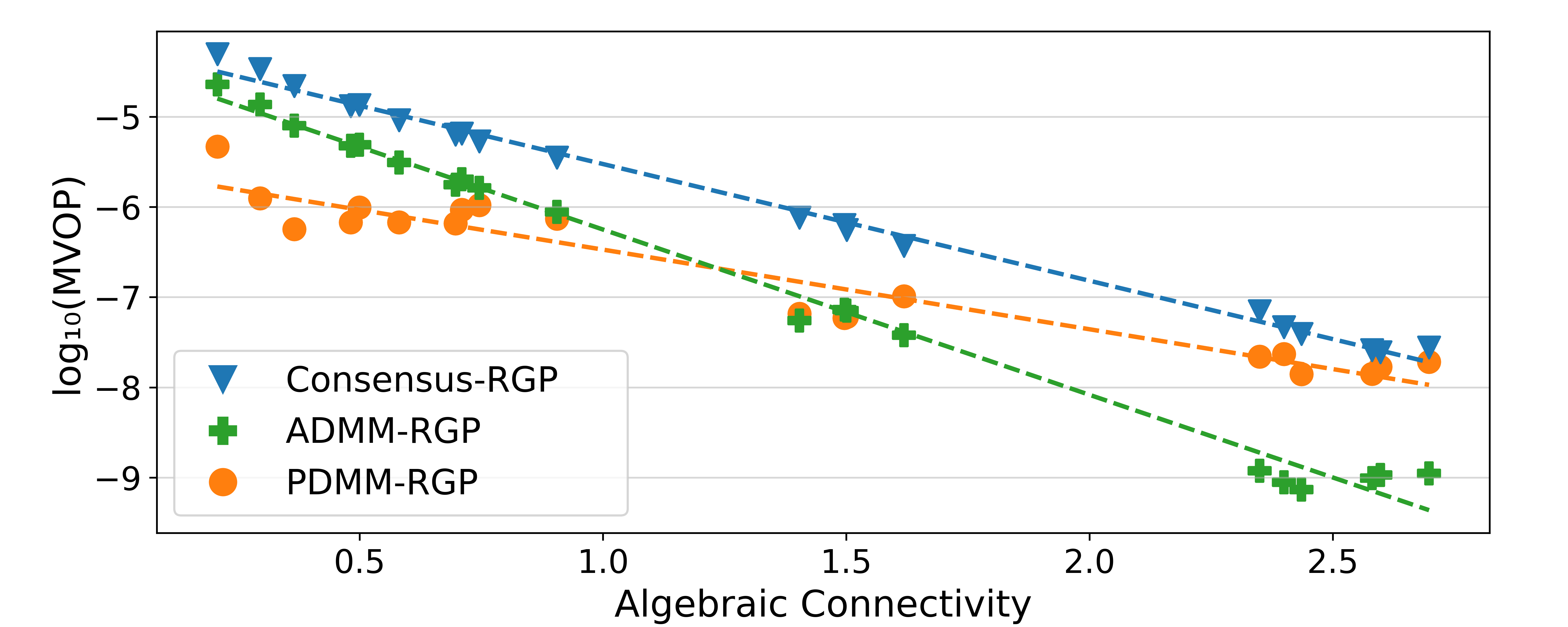}\put(35.6,15.15){\fontfamily{DejaVuSans-TLF}\fontsize{5.2pt}{7pt}\selectfont{\cite{rao_consensus-based_2026}}}
    \end{overpic}
    \caption{Plot of consensus metric MVOP vs. algebraic connectivity of the communication graph with $T=20$.}
    \label{fig:alg_conn_sweep}
    \vspace{-0.4cm}
\end{figure}

\subsection{Results and Discussion}

Consensus-RGP, ADMM-RGP, PDMM-RGP, and centralized RGP were all simulated across 100 Monte Carlo (MC) simulations with different random seeds. All four algorithms achieved an average RMSE of $0.0862$ with a 95\% confidence interval of $[0.0860, 0.0864]$. Figure \ref{fig:reconstructed_fields} shows the reconstructed fields from one simulation alongside that of a centralized, non-recursive GP. The high number of basis locations and large grid used here were chosen deliberately to focus on communication efficiency effects, which is the goal of this work. Hence, a systematic sweep over basis location count is therefore outside the scope of this study. The behavior of these methods under a reduced basis location regime — where distributed and centralized solutions are known to diverge — has been characterized in \cite{rao_consensus-based_2026}.

Figure \ref{fig:ksweep} shows the average $\log_{10}(\text{MVOP})$ versus $K$, the number of communication rounds per time step, over 100 MC runs. For communication graph 1 (top), ADMM-RGP achieves lower MVOP values than the other algorithms for any given $K$. In particular, ADMM-RGP with $K=7$ achieves a consensus level comparable to Consensus-RGP with $K=10$, reducing communication by 30\% without sacrificing performance. Although PDMM-RGP performs less favorably on graph 1, it outperforms the other algorithms on graph 2, achieving MVOP values up to three orders of magnitude lower than Consensus-RGP. On graph 2, both PDMM-RGP and ADMM-RGP achieve performance with $K=5$ comparable to Consensus-RGP with $K=10$, corresponding to a 50\% reduction in communication. Figure \ref{fig:alg_conn_sweep} plots the average $\log_{10}(\text{MVOP})$ over 100 MC runs versus the algebraic connectivity of 20 different communication graphs. PDMM-RGP outperforms the other algorithms in sparse graphs, while ADMM-RGP consistently outperforms Consensus-RGP and performs particularly well in densely connected graphs.
Figure \ref{fig:comp_time} shows box plots of computation time per communication round over 250 rounds as a function of the agent's number of neighbors. All three algorithms exhibit linear scaling, although PDMM-RGP has a steeper increase. While Consensus-RGP has the lowest per-round computational complexity, ADMM-RGP and PDMM-RGP reduce communication and can reduce total computation by requiring fewer rounds.

\begin{figure}[t]
    \centering
    \includegraphics[width=0.9\linewidth]{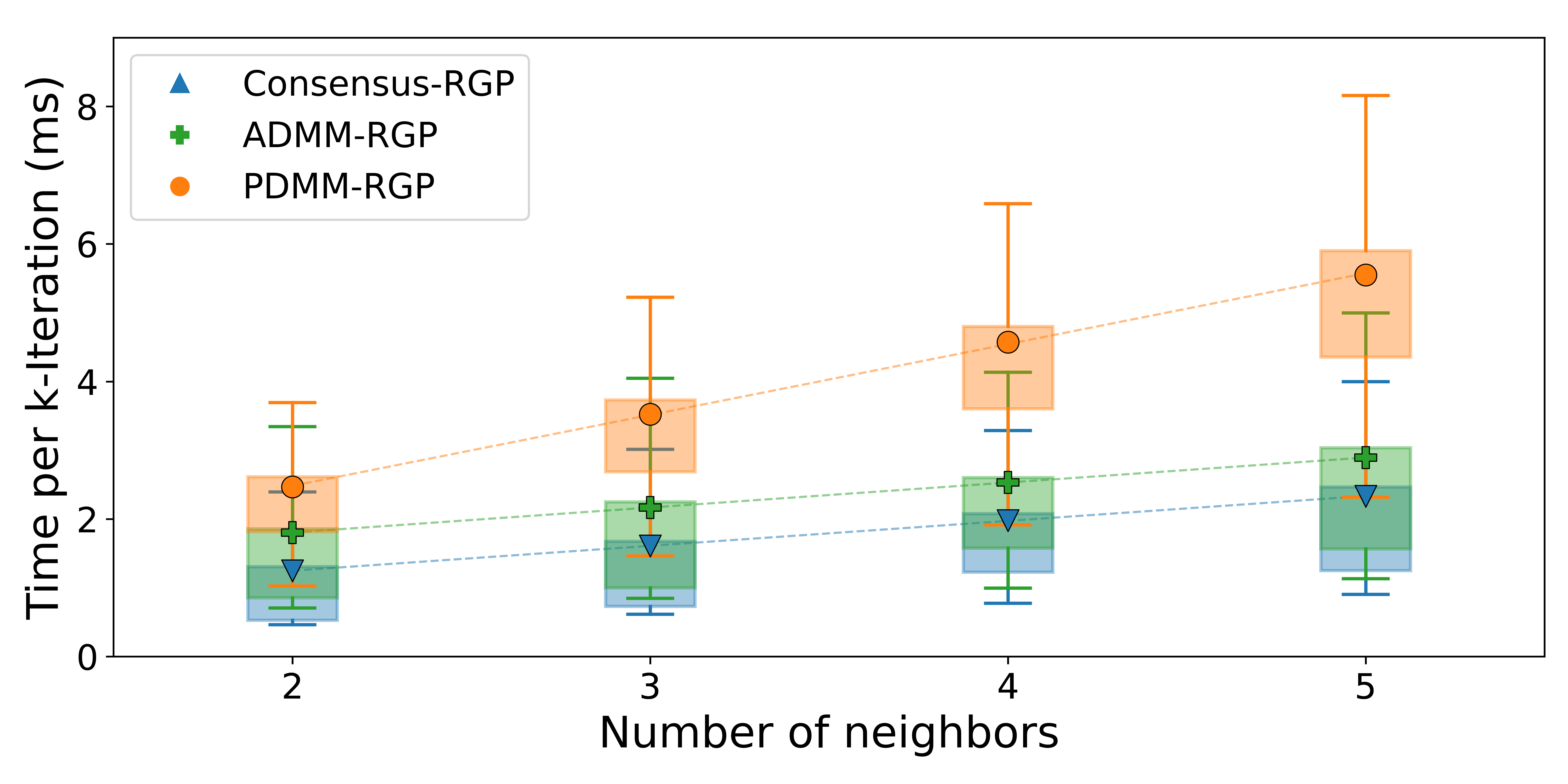}
    \caption{Box plots of computation time from 250 communication rounds using communication graph 1. Markers denote the median times.}
    \label{fig:comp_time}
    \vspace{-0.6cm}
\end{figure}

\begin{figure}

\end{figure}

\section{Conclusion} \label{sec:conclusions}
This paper presented two distributed, recursive, multi-output GP algorithms with reduced communication and computational complexity. We analyzed their convergence and provided practical guidance for parameter selection to ensure fast convergence. Experiments show the proposed algorithms require fewer resources than the state of the art while achieving comparable reconstruction performance, with PDMM-RGP more effective on sparse graphs and ADMM-RGP on denser graphs. Future work will consider asynchronous communication, time-varying functions and communication graphs, and extensions to other GP representations such as random-feature and multiple-model formulations. The proposed fusion framework may also be applied to distributed trajectory and dynamical state estimation, as well as active sensing and cooperative control problems.

\appendices

\section{Proof of Theorem \ref{theorem:admm_stability}} \nonumber \label{sec:theorem_proof}

Using the property $\text{det}(\mathbf{A}) = \prod_i\lambda_i(\mathbf{A})$ \cite{petersen_matrix_2012}, we have
$
\det(\mathbf{M}_i)=-\tau\lambda_i(\mathbf{L}) = \lambda_1(\mathbf{M}_i)\lambda_2(\mathbf{M}_i)$.
Because $\lambda_i(\mathbf{L})$ is positive, $0< -\tau\lambda_i(\mathbf{L}) < 1$. If the eigenvalues of $\mathbf{M}_i$ are real, the characteristic equation,
\begin{equation*}
p_i(\lambda) = \lambda^2 - (1-(\alpha+\tau)\lambda_i(\mathbf{L})) \lambda - \tau \lambda_i(\mathbf L)=0,
\end{equation*}
has two real roots, which satisfy $|\lambda_1\lambda_2| = |-\tau\lambda_i(\mathbf{L})| <1$. Therefore, both roots have magnitude $<1$ provided that neither one crosses the $\pm 1$ boundary. Evaluating the characteristic equation at $1$ and $-1$, we find
\begin{equation*}
p_i(1)=\alpha \lambda_i(\mathbf L) > 0,
\qquad
p_i(-1)=2-(\alpha+2\tau)\lambda_i(\mathbf L)>0,
\end{equation*}
under the assumed bounds on $\tau$. Therefore, both eigenvalues lie strictly within $(-1,1)$. In the case that the eigenvalues are complex conjugate pairs of the form $a\pm bj$, we have $
\det(\mathbf{M}_i)=(a + bj)(a-bj)= a^2 +b^2
$ with squared magnitude $
    |\lambda_1(\mathbf{M}_i)|^2 = |\lambda_2(\mathbf{M}_i)|^2 = a^2+b^2 
$.
Therefore, if the eigenvalues of $\mathbf{M}_i$ are complex, they have squared magnitudes equal to $\det(\mathbf{M}_i) = -\tau\lambda_i(\mathbf{L})$, which we have established is in the range $(0,1)$. Thus, we have shown that the eigenvalues of $\mathbf{M}_i$ lie strictly inside of the unit circle and thus $\mathbf{M}$ is Schur stable.

\ifCLASSOPTIONcaptionsoff
  \newpage
\fi

\newpage
\bibliographystyle{IEEEtran}
\bibliography{references}

\vfill\pagebreak
\end{document}